\documentclass{article}

    \PassOptionsToPackage{numbers, compress}{natbib}

\usepackage{amsmath,amsfonts,amssymb}
\usepackage{amsthm}  {
      \newtheorem{definition}{Definition}
      \newtheorem{lemma}{Lemma}
      \newtheorem{theorem}{Theorem}

      \newtheorem{observation}{Observation}
}

\usepackage{tikz}
\usetikzlibrary{shapes,backgrounds}
\usetikzlibrary{positioning}
\usetikzlibrary{arrows}
\usetikzlibrary{fit}

\DeclareMathOperator*{\argmax}{argmax}

\usepackage[final]{neurips_2020}

\usepackage[utf8]{inputenc} 
\usepackage[T1]{fontenc}    
\usepackage{hyperref}       
\usepackage{url}            
\usepackage{booktabs}       
\usepackage{amsfonts}       
\usepackage{nicefrac}       
\usepackage{microtype}      

\usepackage{algpseudocode}
\usepackage{algorithm}

\algnewcommand{\IfThen}[2]{
  \State \algorithmicif\ #1\ \algorithmicthen\ #2}

\usepackage{graphicx}
\usepackage{wrapfig}

\title{Multi-agent active perception with prediction rewards}

\author{%
  Mikko Lauri \\
  Department of Computer Science\\
  Universit\"{a}t Hamburg\\
  Hamburg, Germany \\
  \texttt{lauri@informatik.uni-hamburg.de} \\
  \And
  Frans A.~Oliehoek \\
  Department of Computer Science\\
  TU Delft\\
  Delft, the Netherlands \\
  \texttt{f.a.oliehoek@tudelft.nl} \\
}

\begin{document}

\maketitle

\begin{abstract}
Multi-agent active perception is a task where a team of agents cooperatively gathers observations to compute a joint estimate of a hidden variable.
The task is decentralized and the joint estimate can only be computed after the task ends by fusing observations of all agents.
The objective is to maximize the accuracy of the estimate.
The accuracy is quantified by a centralized prediction reward determined by a centralized decision-maker who perceives the observations gathered by all agents after the task ends.
In this paper, we model multi-agent active perception as a decentralized partially observable Markov decision process (Dec-POMDP) with a convex centralized prediction reward.
We prove that by introducing individual prediction actions for each agent, the problem is converted into a standard Dec-POMDP with a decentralized prediction reward.
The loss due to decentralization is bounded, and we give a sufficient condition for when it is zero.
Our results allow application of any Dec-POMDP solution algorithm to multi-agent active perception problems, and enable planning to reduce uncertainty without explicit computation of joint estimates.
We demonstrate the empirical usefulness of our results by applying a standard Dec-POMDP algorithm to multi-agent active perception problems, showing increased scalability in the planning horizon.
\end{abstract}

\section{Introduction}
\emph{Active perception}, collecting observations to reduce uncertainty about a hidden variable, is one of the fundamental capabilities of an intelligent agent~\citep{Bajcsy2018}.
In \emph{multi-agent active perception} a team of autonomous agents cooperatively gathers observations to infer the value of a hidden variable.
Application domains include search and rescue robotics, sensor networks, and distributed hypothesis testing.
A multi-agent active perception task often has a finite duration: after observations have been gathered, they are collected to a central database for inference.
While the inference phase is centralized, the observation gathering phase is decentralized: each agent acts independently, without knowing the observations collected by the other agents nor guaranteed communication to the other agents.

The key problem in multi-agent active perception is to determine how each agent should act during the decentralized phase to maximize the informativeness of the collected observations, evaluated afterwards during the centralized inference phase.
The problem can be formalized as a decentralized partially observable Markov decision process (Dec-POMDP)~\citep{Bernstein2002,Oliehoek2016}, a general model of sequential multi-agent decision-making under uncertainty.
At each time step in a Dec-POMDP, each agent in the team takes an individual action.
The next state is determined according to a Markov chain conditional on the current hidden state and all individual actions.
Each agent then perceives an individual observation correlated with the next state and the individual actions.
The agents should act so as to maximize the expected sum of shared rewards, accumulated at each time step over a finite horizon.

In the decentralized phase, the per-step reward depends on the hidden state and the individual actions of all agents.
In typical Dec-POMDPs, the reward on all time steps is of this form.
Active perception problems are modelled by a reward that is a convex function of the team's joint estimate of the hidden state, for example the negative entropy~\citep{Lauri:2019:IGD:3306127.3331815}.
This encourages agents to act in ways that lead to joint state estimates with low uncertainty.
In analogy to the centralized inference phase, the reward at the final time step can be thought of as a \emph{centralized prediction reward} where a centralized decision-maker perceives the individual action-observation histories of all agents and determines a reward based on the corresponding joint state estimate.
Due to the choice of reward function, algorithms for standard Dec-POMDPs are not applicable to such active perception problems.
Despite the pooling of observations after the end of the task, the problem we target is decentralized. 
We design a strategy for each agent to act independently during the observation gathering phase, without knowing for sure how the others acted or what they perceived.
Consequently, the joint state estimate is not available to any agent during the observation gathering phase.
Strategies executable in a centralized manner are available only if all-to-all communication during task execution is possible, which we do not assume.
As decentralized strategies are a strict subset of centralized strategies, the best decentralized strategy is at most as good as the best centralized strategy~\cite{oliehoek2008optimal}.

In this paper, we show that the convex centralized prediction reward can be converted to a \emph{decentralized prediction reward} that is a function of the hidden state and so-called \emph{individual prediction actions}.
This converts the Dec-POMDP with a convex centralized prediction reward into a standard Dec-POMDP with rewards that depend on the hidden state and actions only.
This enables solving multi-agent active perception problems without explicit computation of joint state estimates applying \emph{any} standard Dec-POMDP algorithm.
We show that the error induced when converting the centralized prediction reward into a decentralized one, the \emph{loss due to decentralization}, is bounded.
We also give a sufficient condition for when the loss is zero, meaning that the problems with the centralized, respectively decentralized, prediction rewards are equivalent.
We prove the empirical usefulness of our results by applying standard Dec-POMDP solution algorithms to active perception problems, demonstrating improved scalability over the state-of-the-art.

The remainder of the paper is organized as follows.
We review related work in Section~\ref{sec:related_work}, and give preliminary definitions for Dec-POMDPs in Section~\ref{sec:problem}.
In Section~\ref{sec:conversion}, we introduce our proposed conversion of a centralized prediction reward to a decentralized prediction reward.
We propose a method to apply standard Dec-POMDP algorithms to multi-agent active perception in Section~\ref{sec:adaptive_prediction_action_search}, and empirically evaluate the method in Section~\ref{sec:experiments}.
Section~\ref{sec:conclusion} concludes the paper.

\section{Related work} 
\label{sec:related_work}
We briefly review possible formulations of multi-agent active perception problems, and then focus on the Dec-POMDP model that provides the most general formulation.

Multi-agent active perception has been formulated as a distributed constraint optimization problem (DCOP), submodular maximization, or as a specialized variant of a partially observable Markov decision process (POMDP).
Probabilistic DCOPs with partial agent knowledge have been applied to signal source localization~\citep{Jain2009,Taylor2010}.
DCOPs with Markovian dynamics have been proposed for target tracking by multiple sensors~\citep{Nguyen2014}. 
DCOPs are a simpler model than Dec-POMDPs, as a fixed communication structure is assumed or the noise in the sensing process is not modelled.
Submodular maximization approaches assume the agents' reward can be stated as a submodular set function, and apply distributed greedy maximization to obtain an approximate solution~\citep{Singh2009,Gharesifard2017,Corah2019}.
Along with the structure of the reward function, inter-agent communication is typically assumed.
Specialized variants of POMDPs may also be applied.
If all-to-all communication without delay during task execution is available, centralized control is possible and the problem can be solved as a multi-agent POMDP~\citep{spaan2009decision}.
Auctioning of POMDP policies can facilitate multi-agent cooperation when agents can communicate~\citep{Capitan2013}.
Best et al.~\cite{best2019dec} propose a decentralized Monte Carlo tree search planner where agents periodically communicate their open-loop plans to each other.

Multi-agent active perception without implicit communication with uncertainty on state transitions and the agents' perception may be modelled as a Dec-$\rho$POMDP~\citep{Lauri2017,Lauri:2019:IGD:3306127.3331815}.
In contrast to standard Dec-POMDPs, the reward function in a Dec-$\rho$POMDP is a convex function of the joint state estimate, for example the negative entropy.
Unfortunately, because of the convex reward, standard Dec-POMDP solution algorithms are not applicable.
Furthermore, the heuristic algorithm proposed in~\citep{Lauri:2019:IGD:3306127.3331815,Lauri_JAAMAS2020} requires explicit computation of reachable joint state estimates.
Computing and storing these joint state estimates adds significant memory overhead to the already high computational burden of solving Dec-POMDPs~\citep{Bernstein2002}.
We address both shortcomings, showing that a Dec-$\rho$POMDP can be converted to a standard Dec-POMDP with linear rewards.
This in turn enables applying standard planning algorithms that do not require computation of reachable joint state estimates, leading to improved scalability.

Our work in the decentralized setting draws inspiration from approaches using prediction rewards in single-agent and centralized scenarios.
Araya-L{\'o}pez et al.~\cite{Araya2010} proposed to tackle single-agent active perception as a $\rho$POMDP with a convex reward.
The related POMDP with information rewards (POMDP-IR) was proposed in~\cite{spaan2015decision}.
The POMDP-IR model adds prediction actions that the agent selects in addition to the usual actions.
Active perception is facilitated by rewarding the agent for correctly predicting the true underlying state.
The equivalence of $\rho$POMDP and POMDP-IR model was later established~\cite{Satsangi2018}.
Recently, Satsangi et al.~\cite{Satsangi2020} propose a reinforcement learning method to solve $\rho$POMDPs taking advantage of the equivalence.
In this paper we prove an analogous equivalence for Dec-POMDPs, by converting a Dec-$\rho$POMDP to a standard Dec-POMDP with individual prediction actions and a decentralized prediction reward.
Unlike in the POMDP case, the conversion does not always result in a perfect equivalence, but is associated with a loss due to decentralization.


\section{Multi-agent active perception as a Dec-POMDP}
\label{sec:problem}
In this section, we review how active perception problems are modelled as Dec-POMDPs.
We also review plan-time sufficient statistics, which allow us to concisely express joint state estimates and value functions of a Dec-POMDP.
We concentrate on the practically relevant active perception problem where a prediction reward at the final time step depends on the joint state estimate.

\subsection{Decentralized POMDPs}
\label{subsec:problem_formulation}
We define a Dec-POMDP where the action and observation spaces along with the reward function are time-dependent, as this will be convenient for our subsequent introduction of prediction actions.
\begin{definition}
\label{def:decpomdp}
A Dec-POMDP is a tuple $\mathcal{M}=\langle h$, $I$, $S$, $b_0$, $\mathcal{A}$, $\mathcal{Z}$, $T$, $\mathcal{R}\rangle$, where
\begin{itemize}
	\item $h \in \mathbb{N}$ is the time horizon of the problem,
	\item $I = \{1, 2, \ldots, n\}$ is a set of $n$ agents,
	\item $S$ is the finite set of states $s$,
	\item $b_0 \in \Delta(S)$ is the initial state distribution at time step $t=0$,
	\item $\mathcal{A}$ is the collection of individual action spaces $A_{i,t}$ for each agent $i\in I$ and time step $t=0, \ldots, h-1$. The tuple $a_t = \langle a_{1,t}, a_{2,t}, \ldots, a_{n,t} \rangle$ of individual actions is called the joint action at time step $t$,
	\item $\mathcal{Z}$ is the collection of individual observation spaces $Z_{i,t}$ for each agent $i\in I$ and time step $t=1, \ldots, h$. The tuple $z_t =\langle z_{1,t}, z_{2,t}, \ldots, z_{n,t}\rangle$ of individual observations is called the joint observation at time step $t$,
	\item $T$ is the dynamics function specifying conditional probabilities $\mathbb{P}(z_{t+1}, s_{t+1} \mid s_t, a_t)$, and
	\item $\mathcal{R}$ is the collection of reward functions $R_t(s_t,a_t)$ for time steps $t=0, \ldots, h-1$.
\end{itemize}
\end{definition}

An admissible solution of a Dec-POMDP is a decentralized joint policy $\pi$, i.e., a tuple $\langle \pi_1, \ldots, \pi_n \rangle$ where the individual policy $\pi_i$ of each agent $i$ maps individual observation sequences $\vec{z}_{i,t} = (z_{i,1}, \ldots, z_{i,t})$ to an individual action.\footnote{$\vec{z}_{i,0} = \emptyset$ as there is no observation at time $t=0$.}
An individual policy is a sequence $\pi_i = (\delta_{i,0}, \ldots, \delta_{i,h-1})$ of individual decision rules that map length-$t$ individual observation sequences to an individual action $\delta_{i,t}(\vec{z}_{i,t})=a_{i,t}$.
A joint decision rule is a tuple $\delta_t = \langle \delta_{1,t}, \ldots, \delta_{n,t}\rangle$ that maps a length-$t$ joint observation sequence $\vec{z}_t = \langle \vec{z}_{1,t}, \ldots, \vec{z}_{n,t}\rangle$ to a joint action $\delta_t(\vec{z}_t) = a_t$.
We shall use notation $\vec{z}_{-i,t}$ to denote the individual observation sequences of all agents \emph{except} $i$.

The objective is to find an optimal joint policy $\pi^*$ that maximizes the expected sum of rewards, that is, $\pi^* = \argmax_{\pi}\mathbb{E}\left[\sum_{t=0}^{h-1}R_t(s_t,\delta_t(\vec{z}_t))\right]$, where the expectation is with respect to the distribution of states and joint observation sequences induced under $\pi$, i.e., $\mathbb{P}(s_0, \ldots, s_{h}, \vec{z}_h \mid b_0, \pi) \triangleq b_0(s_0)\prod_{t=0}^{h-1} T(z_{t+1}, s_{t+1} \mid s_t, \delta_t(\vec{z}_t))$.

In this paper, we are interested in cooperative active perception.
We assume that after the task terminates at time step $h$, the joint observation sequence $\vec{z}_{h}$ is used to compute the conditional distribution $b_h\in\Delta(S)$ over the final state $s_h$, that is, $b_h(s_h) \triangleq \mathbb{P}(s_{h} \mid \vec{z}_h, b_0, \pi)$.
We seek policies that, instead of only maximizing the expected sum of rewards $R_t(s_t, a_t)$, also maximize the informativeness of $b_h$.

\begin{wrapfigure}{r}{0.37\textwidth}
  \centering
  \includegraphics[width=0.32\textwidth]{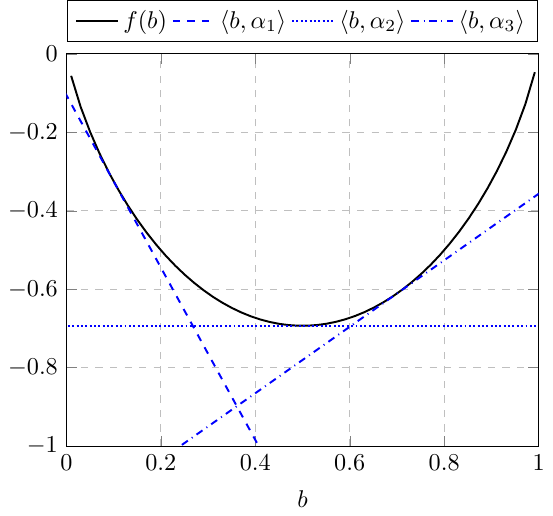}
  \caption{A convex function $f$ is approximated by $\max_{i} \langle b, \alpha_i\rangle$ where $\alpha_i$ are tangent hyperplanes and $\langle \cdot, \cdot\rangle$ is the inner product.
  The horizontal axis depicts $b(s)$ in a system with two states $s$ and $s'$, where $b(s')=1-b(s)$.
  A prediction action corresponds to selecting a single $\alpha_i$.
  }
  \label{fig:approx}
\end{wrapfigure}
The Dec-$\rho$POMDP~\citep{Lauri:2019:IGD:3306127.3331815} models active perception by maximizing a convex function of $b_h$, e.g., the negative entropy.
This convex function may be though of as a \emph{centralized prediction reward} that is independent of the agents' individual actions.
Conceptually, the centralized prediction reward is determined by a virtual centralized decision-maker that perceives the agents' individual action and observation sequences at time $h$, computes $b_h$, and then determines the final reward.
We approximate the centralized prediction reward using a piecewise linear and convex function as illustrated in Figure~\ref{fig:approx}.
Consider a bounded, convex, and differentiable function $f\colon\Delta(S) \to \mathbb{R}$.
The tangent hyperplane $\alpha\in\mathbb{R}^{|S|}$ of $f$ at $b\in \Delta(S)$ is $\alpha = \nabla f(b) - f^*(\nabla f(b))$, where $\nabla$ denotes the gradient and $f^*$ is the convex conjugate of $f$~\citep{boyd2004convex}.
We select a finite set of linearization points $b_j \in \Delta(S)$, and define $\Gamma$ as the set of corresponding tangent hyperplanes $\alpha_j$.\footnote{We defer discussion on how we select the linearization points to Section~\ref{sec:adaptive_prediction_action_search}.}
Then, we obtain a lower approximation as $f(b) \geq \max_{\alpha\in\Gamma} \sum_{s}b(s)\alpha(s)$.
This approximation is also used in $\rho$POMDPs, and has a bounded error~\citep{Araya2010}.
The approximation is also used in~\citep{Satsangi2020}, where error bounds for cases such as 0-1 rewards are provided.
The Dec-$\rho$POMDP problem we consider is as follows.
\begin{definition}[Dec-$\rho$POMDP]
\label{def:decrhopomdp}
A Dec-$\rho$POMDP is a pair $\langle \mathcal{M}, \Gamma\rangle$, where $\mathcal{M}$ is a Dec-POMDP and $\Gamma$ is a set of tangent hyperplanes that determine the centralized prediction reward $\rho\colon\Delta(S) \to \mathbb{R}$ defined as $\rho(b) \triangleq \max\limits_{\alpha \in \Gamma} \sum\limits_{s} b(s)\alpha(s)$.
\end{definition}
The standard Dec-POMDP is the special case where $\Gamma$ contains a single all-zero element.
The objective in the Dec-$\rho$POMDP is to find an optimal joint policy that maximizes the expected sum of rewards and the centralized prediction reward, i.e., $\mathbb{E}\left[\sum_{t=0}^{h-1}R_t(s_t,\delta_t(\vec{z}_t)) + \rho(b_h)\right]$.

\subsection{Sufficient plan-time statistics and the optimal value function} 
\label{sub:sufficient_plan_time_statistics}
Sufficient plan-time statistics are probability distributions over states and joint observation sequences given the past joint policy followed by the agents~\citep{Oliehoek13IJCAI}.
A past joint policy at time $t$ is a joint policy specified until time $t$, denoted $\varphi_t= \langle \varphi_{1,t}, \ldots, \varphi_{n,t} \rangle$, where each individual past policy is a sequence of individual decision rules: $\varphi_{i,t}=(\delta_{i,0}, \ldots, \delta_{i,t-1})$.
The sufficient plan-time statistic for initial state distribution $b_0$ and a past joint policy $\varphi_t$ is defined as
\begin{equation}
\sigma_t(s_t,\vec{z}_t) \triangleq \mathbb{P}(s_t, \vec{z}_t \mid b_0, \varphi_t),
\end{equation}
and at the starting time, $\sigma_0(s_0) \triangleq b_0(s_0)$.
The conditional $\sigma_t(\cdot \mid \vec{z}_t)$ is the state distribution after perceiving the joint observation sequence $\vec{z}_t$ when executing policy $\varphi_t$ with initial state distribution $b_0$.
The marginal $\sigma_t(\vec{z}_t)$ is the probability of the joint observation sequence $\vec{z}_t$ given $\varphi_t$ and $b_0$.

Sufficient plan-time statistics are updateable to any extension of the past joint policy $\varphi_t$.
The extension of an individual past policy $\varphi_{i,t}$ by an individual decision rule $\delta_{i,t}$ is defined $\varphi_{i,t+1} = \varphi_{i,t} \circ \delta_{i,t} \triangleq (\delta_{i,0}, \ldots, \delta_{i,t-1}, \delta_{i,t})$.
The extension of a past joint policy $\varphi_t$ by $\delta_t = \langle \delta_{1,t}, \ldots, \delta_{n,t}\rangle$ is defined $\varphi_{t+1} = \varphi_t \circ \delta_t \triangleq \langle \varphi_{1,t+1}, \ldots, \varphi_{n,t+1} \rangle$.
Let $\vec{z}_{t+1} = (\vec{z}_t, z_{t+1})$ be a joint observation sequence that extends $\vec{z}_t$ with $z_{t+1}$.
Given the statistic $\sigma_t$ for $\varphi_t$ and the next joint decision rule $\delta_t$, the updated statistic for $\varphi_{t+1} = \varphi_t \circ \delta_t$ is $\sigma_{t+1} = U_{ss}(\sigma_t, \delta_t)$, where the update operator $U_{ss}$ is defined as 
\begin{equation}
\label{eq:ss_update}
		[U_{ss}(\sigma_t, \delta_t)](s_{t+1}, \vec{z}_{t+1}) \triangleq 
		\sum\limits_{s_t} T(z_{t+1}, s_{t+1}\mid s_t, \delta_t(\vec{z}_t))\sigma_t(s_t,\vec{z}_t).
\end{equation}

The statistic $\sigma_t$ is sufficient to predict the expected immediate reward $\hat{R}_t(\sigma_t, \delta_t)$ of choosing the next joint decision rule $\delta_t$:
\begin{equation}
\label{eq:plan-time-reward}
\hat{R}_t(\sigma_t, \delta_t) = \sum\limits_{\vec{z}_t, s_t} \sigma_t(s_t, \vec{z}_t) R_t(s_t, \delta_t(\vec{z}_t)).
\end{equation}
Furthermore, $\sigma_h$ is sufficient to predict the expected centralized prediction reward $\hat{\rho}(\sigma_h)$:
\begin{equation}
\label{eq:plan-time-final-reward}
\hat{\rho}(\sigma_h) = \sum\limits_{\vec{z}_h}\sigma_h(\vec{z}_h) \max\limits_{\alpha \in \Gamma} \sum\limits_{s_h} \sigma_h(s_h \mid \vec{z}_h) \alpha(s_h).
\end{equation}
In contrast to the reward in earlier stages, $\hat{\rho}$ does not depend on decision rules at all.
We interpret $\hat{\rho}$ as the expectation of the prediction reward of a centralized decision-maker who selects the best $\alpha \in \Gamma$ after perceiving the full history of joint actions and observations at task termination.

Let $\pi$ be a joint policy consisting of the joint decision rules $\delta_0, \delta_1, \ldots, \delta_{h-1}$.
The value-to-go of $\pi$ starting from a sufficient plan-time statistic $\sigma_t$ at time step $t$ is
\begin{equation}
  Q_t^\pi(\sigma_t, \delta_t) = \begin{cases}
    \hat{R}_t(\sigma_t, \delta_t) + \hat{\rho}(\sigma_{t+1}) & \text{if }t = h-1,\\
    \hat{R}_t(\sigma_t, \delta_t) + Q_{t+1}^\pi(\sigma_{t+1}, \delta_{t+1}) & \text{if } 0 \leq t < h-1,
  \end{cases}
\end{equation}
with the shorthand $\sigma_{t+1} = U_{ss}(\sigma_t,\delta_t)$.
The value function $V^\pi$ of a joint policy $\pi$ is defined as the sum of expected rewards when the agents act according to $\pi$, $V^\pi(\sigma_0) \triangleq Q_0^\pi(\sigma_0, \delta_0)$.
The value function of an optimal policy $\pi^*$ is denoted $V^*$, and it satisfies $V^*(\sigma_0) \geq V^\pi(\sigma_0)$ for all $\pi$.
We write $V^\pi_{\mathcal{M}}$ and $V^{\pi}_{\langle\mathcal{M},\Gamma\rangle}$ for the value function in a standard Dec-POMDP and a Dec-$\rho$POMDP, respectively.

\section{Conversion to standard Dec-POMDP}
\label{sec:conversion}
In this section, we show that any Dec-$\rho$POMDP can be converted to a standard Dec-POMDP by adding \emph{individual prediction actions} for each agent, and by introducing a \emph{decentralized prediction reward}.
Each individual prediction action corresponds to a selection of a tangent hyperplane of the centralized prediction reward, as illustrated in Figure~\ref{fig:approx}.
As the individual prediction actions are chosen in a decentralized manner, the decentralized prediction reward never exceeds the centralized prediction reward.
Consequently, we prove that an optimal policy of the standard Dec-POMDP may be applied to the Dec-$\rho$POMDP with bounded error compared to the true optimal policy.
We give a sufficient condition for when this loss due to decentralization is zero and a Dec-$\rho$POMDP is equivalent to a standard Dec-POMDP.

A major implication of our results is that it is possible to apply any Dec-POMDP solver to a Dec-$\rho$POMDP problem.
We approximate a centralized prediction reward by a decentralized prediction reward that only depends on states and actions.
This allows planning for active perception problems without explicit computation of joint state estimates.

We first introduce our proposed conversion and prove its properties, including the error bound.
We conclude the section by giving a sufficient condition for when the two problems are equivalent.
All omitted proofs are found in the supplementary material.

\begin{definition}
\label{def:decpomdp_conversion}
	Given a Dec-$\rho$POMDP $\langle\mathcal{M}, \Gamma\rangle$ with $\mathcal{M}=\langle h$, $I$, $S$, $b_0$, $\mathcal{A}$, $\mathcal{Z}$, $T$, $\mathcal{R}\rangle$, convert it to a standard Dec-POMDP $\mathcal{M}^+=\langle h+1$, $I$, $S$, $b_0$, $\mathcal{A}^+$, $\mathcal{Z}^+$, $T^+$, $\mathcal{R}^+\rangle$ where the horizon is incremented by one and
	\begin{itemize}
		\item in $\mathcal{A}^+$, the individual action space $A_{i,h}$ for each agent $i\in I$ at time $h$ is a set of \textbf{individual prediction actions} $a_{i,h}$, with one individual prediction action for each tangent hyperplane $\alpha \in \Gamma$; for other time steps $A_{i,t}$ are as in $\mathcal{A}$,
		\item in $\mathcal{Z}^+$, the individual observation space $Z_{i,h+1}$ for each agent $i\in I$ at time $h+1$ contains a single null observation; for other time steps $Z_{i,t}$ are as in $\mathcal{Z}$,
		\item $T^+(z_{h+1}, s_{h+1} \mid s_h, a_h)$ has probability of one for the joint null observation and $s_{h+1} = s_h$, and is otherwise zero; for other time steps $T^+$ is equal to $T$, and
		\item in $\mathcal{R}^+$, the reward function $R_h$ at time step $h$ is a linear combination of \textbf{individual prediction rewards} $R_{i,h}$ of each agent, such that for a joint prediction action $a_h = \langle a_{1,h}, \ldots, a_{n,h} \rangle$ the \textbf{decentralized prediction reward} is
		\begin{equation}
		\label{eq:prediction_reward}
		R_h(s_h, a_h) = \frac{1}{n}\sum\limits_{i=1}^n R_{i,h}(s_h, a_{i,h}),	
		\end{equation}
		with $R_{i,h}(s_h, a_{i,h}) = \alpha_{a_{i,h}}(s_h)$, where $\alpha_{a_{i,h}} \in \Gamma$ is the tangent hyperplane corresponding to the individual prediction action $a_{i,h}$; for other time steps $R_t$ are as in $\mathcal{R}$.
	\end{itemize}
\end{definition}

As the action and observation sets in $\mathcal{M}$ and $\mathcal{M}^+$ are the same until $t=h$, past joint policies are interchangeable, and plan-time sufficient statistics are identical.
\begin{lemma}
\label{lemma:suff_det}
Let $\langle \mathcal{M}, \Gamma \rangle$ be a Dec-$\rho$POMDP and define $\mathcal{M}^+$ as above.
Then, for any past joint policy $\varphi_t$ with $t\leq h$, the respective plan-time sufficient statistics $\sigma_t$ in $\langle \mathcal{M}, \Gamma \rangle$ and $\sigma_t^+$ in $\mathcal{M}^+$ are equivalent, i.e., $\sigma_t \equiv \sigma_t^+$. 
\end{lemma}

In $\mathcal{M}^+$ the horizon is incremented, so that each agent takes one more action than in $\langle \mathcal{M}, \Gamma\rangle$.
The additional action is one of the newly added individual prediction actions at time step $h$.
To select an individual prediction action, an agent may use 1) individual information, that is, the agent's individual observation sequence, and 2) plan-time information common to all agents, that is, the plan-time sufficient statistic.
An \emph{individual prediction rule} $\phi_i$ maps the individual observation sequence $\vec{z}_{i,h}$ and the plan-time sufficient statistic $\sigma_h$ to an individual prediction action.
A \emph{joint prediction rule} $\phi = \langle \phi_1, \ldots, \phi_n \rangle$ is a tuple of individual prediction rules, and maps a joint observation sequence $\vec{z}_h$ and $\sigma_h$ to a joint prediction action.
The \emph{expected decentralized prediction reward} given $\sigma_h$ and $\phi$ is defined analogously to Eq.~\eqref{eq:plan-time-reward} as $\hat{R}_h(\sigma_h, \phi) \triangleq \sum_{\vec{z}_h, s_h}\sigma_h(s_h, \vec{z}_h) R_h(s_h, \phi(\vec{z}_h, \sigma_h))$.

An optimal joint prediction rule maximizes the expected decentralized prediction reward.
We prove that an optimal joint prediction rule consists of individual prediction rules that maximize the expected individual prediction reward.
We then show that the expected decentralized prediction reward is at most equal to the centralized prediction reward, and that a similar relation holds between value functions in the respective Dec-POMDP and Dec-$\rho$POMDP.
The key property required is that the decentralized prediction reward is a sum of individual prediction rewards.

\begin{lemma}[Optimal joint prediction rule]
\label{lemma:predictionrule_optimality}
Let $\langle \mathcal{M}, \Gamma \rangle$ be a Dec-$\rho$POMDP and define $\mathcal{M}^+$ as above, and let $\sigma_h$ be a plan-time sufficient statistic for any past joint policy.
Then the joint prediction rule $\phi^* = \langle \phi_1^*, \ldots, \phi_n^* \rangle$ where each individual prediction rule $\phi_i^*$ is defined as
\begin{equation}
\label{eq:individuale_prediction_rule_definition}
\phi_{i}^*(\vec{z}_{i,h}, \sigma_h) \triangleq \argmax\limits_{a_{i,h} \in A_{i,h}} \sum\limits_{\vec{z}_{-i,h}, s_h}\sigma_h(s_h, \vec{z}_{-i,h}\mid \vec{z}_{i,h}) R_{i,h}(s_h, a_{i,h})
\end{equation}
maximizes expected decentralized prediction reward, that is, $\hat{R}_h(\sigma_h, \phi^*) = \max_{\phi} \hat{R}_h(\sigma_h, \phi)$.
\end{lemma}

In the Dec-$\rho$POMDP $\langle \mathcal{M}, \Gamma \rangle$, the centralized prediction reward is determined by a centralized virtual agent that has access to the observation histories of all agents.
In our converted standard Dec-POMDP $\mathcal{M}^+$, each agent individually takes a prediction action, leading to the decentralized prediction reward.
The decentralized prediction reward never exceeds the centralized prediction reward as we show next.
\begin{lemma}
\label{lemma:lowerbound}
The expected decentralized prediction reward $\hat{R}_h(\sigma_h, \phi^*)$ in $\mathcal{M}^+$ is at most equal to the expected centralized prediction reward $\hat{\rho}(\sigma_h)$ in $\langle \mathcal{M}, \Gamma \rangle$, i.e., $\hat{R}_h(\sigma_h, \phi^*) \leq \hat{\rho}(\sigma_h)$.
\end{lemma}

\begin{lemma}
\label{lemma:valuerelation}
Let $\langle \mathcal{M}, \Gamma\rangle$ and $\mathcal{M}^+$ be as defined above.
Let $\varphi_h$ be a past joint policy for $\mathcal{M}^+$, and let $\phi^*$ be the optimal joint prediction rule.
Then, the value of $\varphi_h \circ \phi^*$ in $\mathcal{M}^+$ is at most equal to the value of $\varphi_h$ in $\langle \mathcal{M}, \Gamma\rangle$, i.e., $V_{\mathcal{M}^+}^{\varphi_h \circ \phi^*}(\sigma_0) \leq V_{\langle \mathcal{M}, \Gamma \rangle}^{\varphi_h}(\sigma_0)$.
\end{lemma}

We now show that the difference between the optimal values of the standard Dec-POMDP $\mathcal{M}^+$ and the Dec-$\rho$POMDP $\langle \mathcal{M}, \Gamma \rangle$ is bounded.
We call the error the \emph{loss due to decentralization}.
\begin{theorem}[Loss due to decentralization]
\label{thm:deterministic_equivalent}
Consider a Dec-$\rho$POMDP $\langle \mathcal{M}, \Gamma \rangle$ with the optimal value function $V_{\langle \mathcal{M}, \Gamma \rangle}^*$.
Let $\pi$ be an optimal policy for the standard Dec-POMDP $\mathcal{M}^+$ created as in Definition~\ref{def:decpomdp_conversion}, and denote by $\varphi_h$ the past joint policy consisting of the first $h$ decision rules of $\pi$. 
Then the difference of $V_{\langle \mathcal{M}, \Gamma \rangle}^*$ and the value function $V_{\langle \mathcal{M}, \Gamma \rangle}^{\varphi_h}$ of applying $\varphi_h$ to $\langle \mathcal{M}, \Gamma \rangle$ is bounded by
\begin{equation}
|V_{\langle \mathcal{M}, \Gamma \rangle}^{*}(\sigma_0) - V_{\langle \mathcal{M}, \Gamma \rangle}^{\varphi_h}(\sigma_0)| \leq 2 \max\limits_{\sigma_h} |\hat{\rho}(\sigma_h) - \hat{R}_h(\sigma_h, \phi^*)|,  
\end{equation}
where $\phi^*$ is the optimal joint prediction rule.
\end{theorem}
\begin{proof}
For clarity, in the following we omit the argument $\sigma_0$ from the value functions.
Suppose the optimal policy $\pi^*$ in $\langle \mathcal{M}, \Gamma\rangle$ consists of the joint decision rules $\delta_0^*, \ldots, \delta_{h-1}^*$.
For $\mathcal{M}^+$, define the partial joint policy $\varphi_h^* = (\delta_0^*, \ldots, \delta_{h-1}^*)$ and its extension $\varphi_h^* \circ \phi^*$.
By Lemma~\ref{lemma:predictionrule_optimality}, an optimal joint policy $\pi$ of $\mathcal{M}^+$ is an extension of some past joint policy $\varphi_h$ by $\phi^*$, i.e., $\pi = \varphi_h \circ \phi^*$.
Because $\varphi_h \circ \phi^*$ is optimal in $\mathcal{M}^+$, we have that $V_{\mathcal{M}^+}^{\varphi_h^*\circ\phi^*} \leq V_{\mathcal{M}^+}^{\varphi_h \circ \phi^*}$.
By Lemma~\ref{lemma:valuerelation}, $V_{\mathcal{M}^+}^{\varphi_h \circ \phi^*} \leq V_{\langle \mathcal{M}, \Gamma \rangle}^{\varphi_h}$.
Finally, because $\pi^*$ is optimal in $\langle\mathcal{M}, \Gamma\rangle$, we have $V_{\langle \mathcal{M}, \Gamma \rangle}^{\varphi_h} \leq V_{\langle \mathcal{M}, \Gamma \rangle}^*$.
Now,
\begin{subequations}
\begin{align}
|V_{\langle \mathcal{M}, \Gamma \rangle}^{*} - V_{\mathcal{M}^+}^{\varphi_h}| &\leq |V_{\langle \mathcal{M}, \Gamma \rangle}^* - V_{\mathcal{M}^+}^{\varphi_h^*\circ\phi^*}| + |V_{\mathcal{M}^+}^{\varphi_h^*\circ\phi^*} - V_{\mathcal{M}^+}^{\varphi_h}| \\
&\leq |V_{\langle \mathcal{M}, \Gamma \rangle}^* - V_{\mathcal{M}^+}^{\varphi_h^*\circ\phi^*}| + |V_{\mathcal{M}^+}^{\varphi_h^*\circ\phi^*} - V_{\langle \mathcal{M}, \Gamma \rangle}^*| =  2\cdot|V_{\langle \mathcal{M}, \Gamma \rangle}^* - V_{\mathcal{M}^+}^{\varphi_h^*\circ\phi^*}| \\
&= 2\cdot \left|\left(\sum\limits_{t=0}^{h-1}\hat{R}_t(\sigma_t^*, \delta_t^*) + \hat{\rho}(\sigma_h^*)\right) - \left(\sum\limits_{t=0}^{h-1}\hat{R}_t(\sigma_t^*, \delta_t^*) + \hat{R}_h(\sigma_h^*, \phi^*) \right) \right|\\
& = 2 \cdot |\hat{\rho}(\sigma_h^*) - \hat{R}_h(\sigma_h^*, \phi^*)|\leq 2 \cdot \max\limits_{\sigma_h} |\hat{\rho}(\sigma_h) - \hat{R}_h(\sigma_h, \phi^*)|. \label{eq:proof_final}
\end{align}
\end{subequations}
We first apply the triangle inequality.
The second inequality holds as $V_{\mathcal{M}^+}^{\varphi_h^*\circ \phi^*} \leq V_{\langle \mathcal{M}, \Gamma \rangle}^{\varphi_h} \leq V_{\langle \mathcal{M}, \Gamma \rangle}^{*}$.
The next equality is by symmetry of absolute difference.
The third line follows from the definition of the value function and Lemma~\ref{lemma:suff_det}.
Note that we refer by $\sigma_t^*$ to the sufficient plan-time statistics reached under partial joint policies $\varphi_t^*$ of $\pi^*$.
The final inequality follows by maximizing the difference of the expected centralized and decentralized prediction rewards.
\end{proof}
The theorem shows that given a Dec-$\rho$POMDP $\langle \mathcal{M}, \Gamma \rangle$, we may solve the standard Dec-POMDP $\mathcal{M}^+$ and apply its optimal policy to the Dec-$\rho$POMDP with bounded error.
If there is only a single agent, that agent is equivalent to the conceptualized centralized agent, which implies $\hat{\rho} \equiv \hat{R}_h$ and that the loss is zero.
The equivalence of a $\rho$POMDP with a convex prediction reward and a POMDP with information rewards first shown in~\citep{Satsangi2018} is therefore obtained as a special case.

The proof indicates that the error is at most twice the difference of the expected centralized prediction reward and the expected decentralized prediction reward at the sufficient plan-time statistic $\sigma_h^*$ at time $h$ under an optimal policy $\pi^*$ of $\langle \mathcal{M}, \Gamma \rangle$.
This suggests that the final bound given as maximum over all sufficient plan-time statistics may be overly pessimistic for some cases.
While a complete characterization of such cases is beyond the scope of this paper, we give below a sufficient condition for when the error is zero in the multi-agent case.
The proof is in the supplementary material.

\begin{observation}
\label{obs:suff}
Consider the setting of Theorem~\ref{thm:deterministic_equivalent}.
Assume that the observation sequence of each agent is conditionally independent of the observation sequences of all other agents given the past joint policy and initial state distribution, i.e., for every agent $i$, $\sigma_h(\vec{z}_h) = \sigma_h(\vec{z}_{i,h})\sigma_h(\vec{z}_{-i,h})$.
Then $\pi^*$ is an optimal joint policy for $\langle \mathcal{M}, \Gamma\rangle$ if and only if $\pi^* \circ \phi^*$ is an optimal joint policy for $\mathcal{M}^+$.
\end{observation}
The sufficient condition above is restrictive.
Informally, it requires that each agent executes its own independent active perception task.
However, the loss due to decentralization may be small if a multi-agent active perception task \emph{almost} satisfies the requirement.
It might be possible to derive less restrictive sufficient conditions in weakly-coupled problems by building on influence-based abstractions~\cite{Oliehoek2012influence}
Such weakly-coupled cases may arise, e.g., in distributed settings where two agents deployed in different regions far away from each other have limited influence on each other.

\section{Adaptive prediction action search} 
\label{sec:adaptive_prediction_action_search}
Recall that the centralized prediction reward function $\rho$ is obtained by approximating a continuous convex function $f$ by a set of $\alpha$-vectors.
The approximation is more accurate the more $\alpha$-vectors are used: for error bounds, see~\citep{Araya2010,Satsangi2020}.
We propose a planning algorithm for Dec-$\rho$POMDPs called Adaptive Prediction Action Search (APAS) that dynamically updates the $\alpha$-vectors during planning.
This avoids the need to a priori create a large number of $\alpha$-vectors some of which may turn out to not be useful.
Moreover, APAS allows application of any standard Dec-POMDP algorithm to solve Dec-$\rho$POMDPs by using the conversion proposed in Section~\ref{sec:conversion}.
This can avoid the computation of all reachable state estimates required by earlier algorithms for Dec-$\rho$POMDPs~\citep{Lauri:2019:IGD:3306127.3331815,Lauri_JAAMAS2020}.

\begin{algorithm}[t]
\footnotesize
\caption{Adaptive prediction action search (APAS) for Dec-$\rho$POMDP planning}
\label{alg:apas}
\begin{algorithmic}[1]
\Require{Dec-$\rho$POMDP $\langle \mathcal{M}, \Gamma \rangle$, convex function $f\colon\Delta(S)\to\mathbb{R}$, number of linearization points $K$}
\Ensure{Best joint policy found, $\pi_{best}$}
  \State $V_{best} \gets -\infty, \pi_{best} \gets \emptyset$
  \Repeat
    \State \textit{// Policy optimization phase}
    \State $\mathcal{M}^+ \gets$ \Call{ConvertDecPOMDP}{$\mathcal{M}, \Gamma$} \Comment{Apply Definition~\ref{def:decpomdp_conversion}}
    \State $\pi \gets $ \Call{Plan}{$\mathcal{M}^+$} \Comment{Use any Dec-POMDP planner}
    \State $V \gets $ \Call{Evaluate}{$\pi$} \label{line:evaluate}
    \IfThen{$V > V_{best}$}{$V_{best} \gets V, \pi_{best} \gets \pi$}
    \State \textit{// Adaptation phase}
    \State $\Gamma \gets \emptyset$ \label{line:adapt_start}
    \For{$k = 1, \ldots, K$} 
      \State $\vec{z}_h \sim \sigma_h(\vec{z}_h)$ \Comment{Sample joint observation sequence $\vec{z}_h$ using $\pi_{best}$} \label{line:joint_obs_seq}
      \State $b_k \gets \sigma_h(\cdot \mid \vec{z}_h)$ \Comment{Final state estimate corresponding to $\vec{z}_h$} \label{line:state_estimate}
      \State $\alpha_k \gets \nabla f(b_k) - f^*(\nabla f(b_k))$ \Comment{Tangent hyperplane of $f$ at $b_k$} \label{line:alpha_vector}
      \State $\Gamma \gets \Gamma \cup \{\alpha_k\}$
    \EndFor \label{line:adapt_end} 
  \Until{converged}
  \State \Return $\pi_{best}$
\end{algorithmic}
\end{algorithm}

The pseudocode for APAS is shown in Algorithm~\ref{alg:apas}.
APAS consists of two phases that are repeated until convergence: the policy optimization phase and the $\alpha$-vector adaptation phase.
In the policy optimization phase, the best joint policy for the current set $\Gamma$ of $\alpha$-vectors is found.
On the first iteration we initialize $\Gamma$ randomly as explained in the next section.
In the adaptation phase, the $\alpha$-vectors are then modified such that they best approximate the final reward for joint state estimates that are most likely reached by the current joint policy.
In the optimization phase the Dec-$\rho$POMDP is converted to a standard Dec-POMDP as described in Section~\ref{sec:conversion}.
We then apply any standard Dec-POMDP algorithm to plan a joint policy $\pi$ for the converted Dec-POMDP.
If the value of $\pi$ exceeds the value of the best policy so far, $\pi_{best}$, the best policy is updated.
In the adaptation phase, we sample a final state distribution $b_k$ under the currently best policy, and insert the tangent hyperplane at $b_k$ into $\Gamma$.
This corresponds to sampling a point on the horizontal axis in Figure~\ref{fig:approx} and adding the corresponding tangent hyperplane to $\Gamma$.
This is repeated $K$ times.
Specifically, we simulate a trajectory under the current best policy $\pi_{best}$ to sample a joint observation sequence (Line~\ref{line:joint_obs_seq}), use Bayesian filtering to compute the corresponding state estimate $b_k$ (Line~\ref{line:state_estimate}), and compute corresponding $\alpha$-vector (Line~\ref{line:alpha_vector}).
The updated set $\Gamma$ of $\alpha$-vectors provides the best approximation of the final reward for the sampled state estimates.
The time complexity of APAS is determined by the complexity of the Dec-POMDP planner called in $\textsc{Plan}$.
A reference implementation is available at \url{https://github.com/laurimi/multiagent-prediction-reward}.

\section{Experiments} 
\label{sec:experiments}
The Dec-$\rho$POMDP we target is computationally more challenging (NEXP-complete~\citep{Bernstein2002}) than centralized POMDP, POMDP-IR, or  $\rho$POMDP (PSPACE-complete~\citep{papadimitriou1987complexity}).
Immediate all-to-all communication during task execution would be required to solve the problem we target as a centralized problem, and the resulting optimal value would be higher~\citep{oliehoek2008optimal}.
Therefore, a comparison to centralized methods is neither fair nor necessary.
We compare APAS to the NPGI algorithm of~\cite{Lauri:2019:IGD:3306127.3331815} that solves a Dec-$\rho$POMDP by iterative improvement of a fixed-size policy graph.
As the \textsc{Plan} subroutine of APAS, we use the finite-horizon variant of the policy graph improvement method of~\cite{Pajarinen2011}.
This method is algorithmically similar to NPGI which helps isolate the effect of applying the proposed conversion to a standard Dec-POMDP from the effect of algorithmic design choices to the extent possible.
Any other Dec-POMDP algorithm may also be used with APAS.
Details on the algorithms and parameter settings are provided in the supplementary material.

We evaluate on the Dec-$\rho$POMDP domains from~\cite{Lauri:2019:IGD:3306127.3331815}: the micro air vehicle (MAV) domain and information gathering rovers domain.
In the MAV domain two agents cooperatively track a moving target to determine if the target is friendly or hostile.
In the rovers domain, two agents collect samples of scientific data in a grid environment.
The reward at the final time step is the negative entropy of the joint state estimate $b_h$, that is, $f(b_h) = \sum_{s_h}b_h(s_h)\ln b_h(s_h)$.
For the MAV problem, we use $K=2$, and for the rovers problem $K=5$ individual prediction actions.
For APAS, we initialize $\Gamma$ by randomly sampling linearization points $b_k \in \Delta(S)$, $k=1,\ldots, K$ using~\cite{smith2004sampling}.
The corresponding $\alpha$-vectors for the negative entropy are $\alpha_k(s) = \ln b_k(s)$.
We only observed minor benefits from using a greater number of $\alpha$-vectors and only for long horizons, details are reported in the supplementary material.
We run 100 repetitions using APAS and report the average policy value and its standard error.
Results for NPGI are as reported in~\cite{Lauri_JAAMAS2020}.
To ensure comparability, all reported policy values are computed using $f$ as the final reward.
As a baseline, we report the optimal value if known.

\begin{table}[t]
\scriptsize
\caption{Average policy values $\pm$ standard error in the MAV (left) and the rovers domains (right).}
\label{tab:results}
\begin{tabular}{@{}lllll@{}}
\multicolumn{5}{c}{MAV}                                        \\ 
\toprule
$h$ & APAS (ours) & APAS (no adapt.) & NPGI~\citep{Lauri_JAAMAS2020} & Optimal \\ \midrule
2         & -2.006 $\pm$ 0.005 & -2.112 $\pm$ 0.002 & -1.931     & -1.919        \\
3         & -1.936 $\pm$ 0.004 & -2.002 $\pm$ 0.002 & -1.833     & -1.831        \\
4         & -1.879 $\pm$ 0.004 & -1.943 $\pm$ 0.002 & -1.768     & ---        \\
5         & -1.842 $\pm$ 0.005 & -1.918 $\pm$ 0.002 & -1.725     & ---        \\
6         & -1.814 $\pm$ 0.004 & -1.898 $\pm$ 0.002 &  ---    &  ---       \\
7         & -1.789 $\pm$ 0.004 & -1.892 $\pm$ 0.003 &  ---    &  ---       \\
8         & -1.820 $\pm$ 0.005 & -1.885 $\pm$ 0.006 & ---     &  ---       \\
\bottomrule
\end{tabular}
\quad
\begin{tabular}{@{}lllll@{}}
\multicolumn{5}{c}{Rovers}                                        \\ 
\toprule
$h$ & APAS (ours) & APAS (no adapt.) & NPGI~\citep{Lauri_JAAMAS2020} & Optimal \\ \midrule
2         & -3.484 $\pm$ 0.002 & -3.800 $\pm$ 0.006 & -3.495     & -3.479        \\
3         & -3.402 $\pm$ 0.008 & -3.680 $\pm$ 0.006 & -3.192     & -3.189        \\
4         & -3.367 $\pm$ 0.011 & -3.640 $\pm$ 0.006 & -3.036     &  ---       \\
5         & -3.293 $\pm$ 0.012 & -3.591 $\pm$ 0.006 & -2.981     &  ---       \\
6         & -3.333 $\pm$ 0.012 & -3.631 $\pm$ 0.007 & ---     &  ---       \\
7         & -3.375 $\pm$ 0.014 & -3.673 $\pm$ 0.008 & ---     &  ---       \\
8         & -3.496 $\pm$ 0.014 & -3.860 $\pm$ 0.010 & ---     &  ---       \\
\bottomrule
\end{tabular}
\end{table}

\paragraph{Comparison to the state-of-the-art.}
The results are shown in Table~\ref{tab:results}.
While APAS does not exceed state-of-the-art performance in terms of policy quality, the major advantage of APAS is that it is able to scale to greater horizons than NPGI.
The number of joint state distributions reachable under the current policy grows exponentially in the horizon $h$.
Computation of these state distributions is required by NPGI, causing it to run out of memory with $h>5$.

While running APAS, we compute the value of the policies exactly (Alg.~\ref{alg:apas}, Line~\ref{line:evaluate}).
For long horizons, exact evaluation requires a significant fraction of the computation budget.
We expect that further scaling in terms of planning horizon is possible by switching to approximate evaluation, at the cost of noisy value estimates.
Results on computation time as well as results for horizons up to 10 for the Rovers problem are reported in the supplementary material.

\paragraph{Benefit of adaptive prediction action selection.}
The adaptation phase of APAS is disabled by not running Lines~\ref{line:adapt_start}-\ref{line:adapt_end}.
Instead, at the start of each iteration of the while loop we randomly sample $K$ linearization points using~\cite{smith2004sampling} to create $\Gamma$ and solve the corresponding Dec-POMDP.
We repeat this procedure 1000 times and report the results in the column ``APAS (no adapt.)'' of Table~\ref{tab:results}.
We conclude that the $\alpha$-vector adaptation clearly improves performance in both domains.


\section{Conclusion} 
\label{sec:conclusion}
We showed that multi-agent active perception modelled as a Dec-$\rho$POMDP can be reduced to a standard Dec-POMDP by introducing individual prediction actions.
The difference between the optimal solution of the standard Dec-POMDP and the Dec-$\rho$POMDP is bounded.
Our reduction enables application of any standard Dec-POMDP solver to multi-agent active perception problems, as demonstrated by our proposed APAS algorithm.

Our results allow transferring advances in scalability for standard Dec-POMDPs to multi-agent active perception tasks.
In multi-agent reinforcement learning, rewards typically depend on the underlying state of the system.
Therefore, our reduction result also enables further investigation into learning for multi-agent active perception.
An investigation of the necessary conditions for when the loss due to decentralization is zero is another future direction.


\section*{Broader Impact}
This work is theoretical in nature, and therefore does not present any foreseeable societal consequence.

\section*{Acknowledgments}

\begin{wrapfigure}[4]{r}{0.4\columnwidth}
  \vspace{-2\baselineskip}
  \centering{\includegraphics[width=0.39\columnwidth]{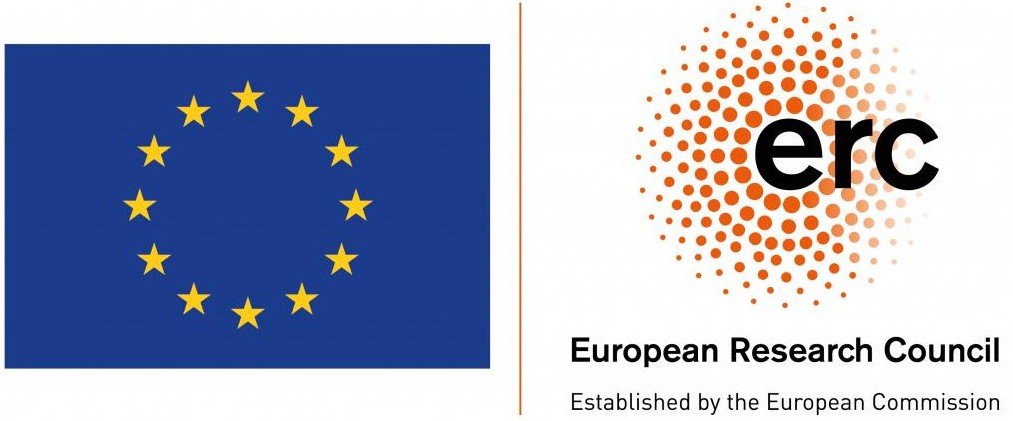}}
\end{wrapfigure}
This project had received funding from the European Research Council (ERC) under the European Union's Horizon 2020 research and innovation programme (grant agreement No.~758824 \textemdash INFLUENCE).

\appendix
\section{Overview of supplementary appendix sections} 
\label{sec:introduction}
The supplementary appendix sections are structured as follows.
We present proofs of the lemmas and the observation in Section~\ref{sec:proofs}.
Additional details on the APAS algorithm and computation of the $\alpha$-vectors are provided in Section~\ref{sec:further_details_on_the_apas_algorithm}.
We give details of our experimental setup in Section~\ref{sec:experimental_setup}.
Finally, additional experimental results are provided in Section~\ref{sec:experimental_results}.

\makeatletter
\@addtoreset{lemma}{section}
\@addtoreset{theorem}{section}
\@addtoreset{observation}{section}
\makeatother
\section{Proofs} 
\label{sec:proofs}
In this section, we present the proofs omitted from the main paper.
Lemmas~\ref{lemma:suff_det}, \ref{lemma:predictionrule_optimality}, \ref{lemma:lowerbound}, and~\ref{lemma:valuerelation} are proven in Subsections~\ref{sub:proof_of_lemma_1}, \ref{sub:proof_of_lemma_2}, \ref{sub:proof_of_lemma_3}, and~\ref{sub:proof_of_lemma_4}, respectively.
Finally, Observation~\ref{obs:suff} is proven in Subsection~\ref{sub:proof_of_corollary_1}.

\subsection{Proof of Lemma~1} 
\label{sub:proof_of_lemma_1}
\begin{lemma}
Let $\langle \mathcal{M}, \Gamma \rangle$ be a Dec-$\rho$POMDP and define $\mathcal{M}^+$ as above.
Then, for any past joint policy $\varphi_t$ with $t\leq h$, the respective plan-time sufficient statistics $\sigma_t$ in $\langle \mathcal{M}, \Gamma \rangle$ and $\sigma_t^+$ in $\mathcal{M}^+$ are equivalent: $\sigma_t \equiv \sigma_t^+$. 
\end{lemma}
\begin{proof}
  By induction.
  As the initial state distributions are equal, clearly $\sigma_0 \equiv \sigma_0^+$.
  Now assume $\sigma_t \equiv \sigma_t^+$ for some $0 \leq t < h$.
  By Definition~\ref{def:decpomdp_conversion} the transition functions $T$ and $T^+$ are equivalent.
  Therefore, the respective sufficient statistic update operators $U_{ss}$ and $U_{ss}^+$ as defined in Eq.~\eqref{eq:ss_update} are equivalent for the next decision rule $\delta_t$.
  We conclude that $\sigma_{t+1} \equiv \sigma_{t+1}^+$.
\end{proof}

\subsection{Proof of Lemma~2} 
\label{sub:proof_of_lemma_2}
\begin{lemma}[Optimal joint prediction rule]
Let $\langle \mathcal{M}, \Gamma \rangle$ be a Dec-$\rho$POMDP and define $\mathcal{M}^+$ as above, and let $\sigma_h$ be a plan-time sufficient statistic for any past joint policy.
Then the joint prediction rule $\phi^* = \langle \phi_1^*, \ldots, \phi_n^* \rangle$ where each individual prediction rule $\phi_i^*$ is defined as
\begin{equation}
\phi_{i}^*(\vec{z}_{i,h}, \sigma_h) \triangleq \argmax\limits_{a_{i,h} \in A_{i,h}} \sum\limits_{\vec{z}_{-i,h}, s_h}\sigma_h(s_h, \vec{z}_{-i,h}\mid \vec{z}_{i,h}) R_{i,h}(s_h, a_{i,h})
\end{equation}
maximizes the expected decentralized prediction reward, i.e, $\hat{R}_h(\sigma_h, \phi^*) = \max_{\phi} \hat{R}_h(\sigma_h, \phi)$.
\end{lemma}
\begin{proof}
We proceed from the definition of the maximum expected decentralized prediction reward:
\begin{subequations}
\label{eq:prediction_optimality_proof}
\begin{align}
\max\limits_{\phi} \hat{R}_h(\sigma_h, \phi) &\triangleq \max\limits_{\phi} \sum\limits_{\vec{z}_h, s_h}\sigma_h(s_h, \vec{z}_h) R_h(s_h, \phi(\vec{z}_h, \sigma_h)) \label{eq:prediction_optimality_proof_a}\\
&= \max\limits_{\phi_1, \ldots, \phi_n} \sum\limits_{\vec{z}_h, s_h}\sigma_h(s_h, \vec{z}_h) \frac{1}{n}\sum\limits_{i=1}^{n}R_{i,h}(s_h, \phi_{i}(\vec{z}_{i,h}, \sigma_h)) \label{eq:prediction_optimality_proof_b}\\
&= \frac{1}{n}\sum\limits_{i=1}^{n} \max\limits_{\phi_i}\sum\limits_{\vec{z}_h, s_h}\sigma_h(s_h, \vec{z}_h)R_{i,h}(s_h, \phi_{i}(\vec{z}_{i,h}, \sigma_h)) \label{eq:prediction_optimality_proof_c} \\
&=\frac{1}{n}\sum\limits_{i=1}^{n} \max\limits_{\phi_i}\sum\limits_{\vec{z}_{i,h}, \vec{z}_{-i,h} s_h} \sigma_h(\vec{z}_{i,h})\sigma_h(s_h, \vec{z}_{-i,h}\mid \vec{z}_{i,h})R_{i,h}(s_h, \phi_{i}(\vec{z}_{i,h}, \sigma_h)) \label{eq:prediction_optimality_proof_d} \\
&=\frac{1}{n}\sum\limits_{i=1}^{n} \sum\limits_{\vec{z}_{i,h}} \sigma_h(\vec{z}_{i,h}) \max\limits_{a_{i,h} \in A_{i,h}} \sum\limits_{\vec{z}_{-i,h},s_h}\sigma_h(s_h, \vec{z}_{-i,h}\mid \vec{z}_{i,h})R_{i,h}(s_h, a_{i,h}) \label{eq:prediction_optimality_proof_e} \\
&=\frac{1}{n}\sum\limits_{i=1}^{n} \sum\limits_{\vec{z}_{i,h}} \sigma_h(\vec{z}_{i,h}) \sum\limits_{\vec{z}_{-i,h},s_h}\sigma_h(s_h, \vec{z}_{-i,h}\mid \vec{z}_{i,h})R_{i,h}(s_h, \phi_i^*(\vec{z}_{i,h}, \sigma_h)) \label{eq:prediction_optimality_proof_f} \\
&= \hat{R}_h(\sigma_h, \phi^*).
\end{align}
\end{subequations}
Above, \ref{eq:prediction_optimality_proof_b} follows by definition of $R_h$ as sum of individual prediction rewards (Eq.~\eqref{eq:prediction_reward}) and since $\phi$ is decentralized.
Then \ref{eq:prediction_optimality_proof_c} and~\ref{eq:prediction_optimality_proof_d} follow by rearranging terms and by law of conditional probability, respectively.
Then~\ref{eq:prediction_optimality_proof_e} follows since maximizing over $\phi_i$ is equivalent to finding an individual prediction action for each $\vec{z}_{i,h}$ that maximizes the expected individual prediction reward.
Equality~\ref{eq:prediction_optimality_proof_f} follows from the definition of $\phi_i^*$, completing the proof.
\end{proof}

\subsection{Proof of Lemma~3} 
\label{sub:proof_of_lemma_3}
\begin{lemma}
The expected decentralized prediction reward $\hat{R}_h(\sigma_h, \phi^*)$ in $\mathcal{M}^+$ is at most equal to the expected centralized prediction reward $\hat{\rho}(\sigma_h)$ in $\langle \mathcal{M}, \Gamma \rangle$, i.e., $\hat{R}_h(\sigma_h, \phi^*) \leq \hat{\rho}(\sigma_h)$.
\end{lemma}
\begin{proof}
We continue from Eq.~\eqref{eq:prediction_optimality_proof_e}:
\begin{subequations}
\label{eq:reward_lowerbound}
\begin{align}
&\hat{R}_h(\sigma_h, \phi^*) = \frac{1}{n}\sum\limits_{i=1}^{n}\sum\limits_{\vec{z}_{i,h}} \sigma_h(\vec{z}_{i,h}) \max\limits_{a_{i,h} \in A_{i,h}}\sum\limits_{\vec{z}_{-i,h}, s_h} \sigma_h(s_h,\vec{z}_{-i,h}\mid \vec{z}_{i,h})R_{i,h}(s_h, a_{i,h})\\
&= \frac{1}{n}\sum\limits_{i=1}^{n}\sum\limits_{\vec{z}_{i,h}} \sigma_h(\vec{z}_{i,h}) \max\limits_{a_{i,h} \in A_{i,h}}\sum\limits_{\vec{z}_{-i,h}} \sigma_h(\vec{z}_{-i,h}\mid \vec{z}_{i,h})\sum\limits_{s_h}\sigma_h(s_h\mid \vec{z}_{-i,h}, \vec{z}_{i,h})R_{i,h}(s_h, a_{i,h}) \label{eq:reward_lowerbound_a}\\
&\leq \frac{1}{n}\sum\limits_{i=1}^{n}\sum\limits_{\vec{z}_{i,h},\vec{z}_{-i,h}} \sigma_h(\vec{z}_{i,h})\sigma_h(\vec{z}_{-i,h}\mid \vec{z}_{i,h}) \max\limits_{a_{i,h} \in A_{i,h}}\sum\limits_{s_h} \sigma_h(s_h\mid \vec{z}_{-i,h}, \vec{z}_{i,h})R_{i,h}(s_h, a_{i,h}) \label{eq:reward_lowerbound_b}\\
&=\frac{1}{n}\sum\limits_{i=1}^{n}\sum\limits_{\vec{z}_{h}} \sigma_h(\vec{z}_{h}) \max\limits_{a_{i,h} \in A_{i,h}}\sum\limits_{s_h} \sigma_h(s_h\mid \vec{z}_{h})R_{i,h}(s_h, a_{i,h}) \label{eq:reward_lowerbound_c}\\
&=\frac{1}{n}\sum\limits_{i=1}^{n}\sum\limits_{\vec{z}_{h}} \sigma_h(\vec{z}_{h}) \max\limits_{\alpha \in \Gamma}\sum\limits_{s_h} \sigma_h(s_h\mid \vec{z}_{h})\alpha(s_h) \label{eq:reward_lowerbound_d}\\
&=\frac{1}{n}\sum\limits_{i=1}^{n}\hat{\rho}(\sigma_h) = \hat{\rho}(\sigma_h).
\end{align}
\end{subequations}
Equality~\ref{eq:reward_lowerbound_a} follows by law of conditional probability and rearranging the terms.
Inequality~\ref{eq:reward_lowerbound_b} follows since the expectation of a maximum is greater than or equal to the maximum of the expectation.
Then, \ref{eq:reward_lowerbound_c} follows by rearranging terms.
Finally, \ref{eq:reward_lowerbound_d} follows due to the one-to-one correspondence between the individual prediction actions $a_{i,h} \in A_{i,h}$ and the tangent hyperplanes $\alpha \in \Gamma$ from Definition~\ref{def:decpomdp_conversion}.
\end{proof}

\subsection{Proof of Lemma~4} 
\label{sub:proof_of_lemma_4}
\begin{lemma}
Let $\langle \mathcal{M}, \Gamma\rangle$ and $\mathcal{M}^+$ be as defined above.
Let $\varphi_h$ be a past joint policy for $\mathcal{M}^+$, and let $\phi^*$ be the optimal joint prediction rule.
Then, the value of $\varphi_h \circ \phi^*$ in $\mathcal{M}^+$ is at most equal to the value of $\varphi_h$ in $\langle \mathcal{M}, \Gamma\rangle$, i.e., $V_{\mathcal{M}^+}^{\varphi_h \circ \phi^*} \leq V_{\langle \mathcal{M}, \Gamma \rangle}^{\varphi_h}$.
\end{lemma}
\begin{proof}
Suppose that $\varphi_h$ consists of the joint decision rules $\delta_0, \ldots, \delta_{h-1}$.
Now $\varphi_h$ can be applied as a full joint policy in $\langle\mathcal{M}, \Gamma\rangle$ with value function $V_{\langle\mathcal{M}, \Gamma\rangle}^{\varphi_h}$.
Since the value function of a policy is defined as the sum of expected rewards, for any initial plan-time sufficient statistic $\sigma_0$,
\begin{equation}
\label{eq:valueproof_1}
V_{\mathcal{M}^+}^{\varphi_h \circ \phi^*}(\sigma_0) = \sum\limits_{t=0}^{h-1}\hat{R}_t(\sigma_t, \delta_t) + \hat{R}_h(\sigma_h, \phi^*) \leq \sum\limits_{t=0}^{h-1}\hat{R}_t(\sigma_t, \delta_t) + \hat{\rho}(\sigma_h) = V_{\langle \mathcal{M}, \Gamma \rangle}^{\varphi_h}(\sigma_0),
\end{equation}
where the inequality follows since by Lemma~\ref{lemma:suff_det} the sufficient plan-time statistics are equivalent and by Lemma~\ref{lemma:lowerbound} the decentralized prediction reward lower bounds the centralized prediction reward.
\end{proof}


\subsection{Proof of Observation~1} 
\label{sub:proof_of_corollary_1}
\begin{observation}
Consider the setting of Theorem~\ref{thm:deterministic_equivalent}.
Assume that the observation sequence of each agent is conditionally independent of the observation sequences of all other agents given the past joint policy and initial state distribution, i.e., for every agent $i$, $\sigma_h(\vec{z}_h) = \sigma_h(\vec{z}_{i,h})\sigma_h(\vec{z}_{-i,h})$.
Then $\pi^*$ is an optimal joint policy for $\langle \mathcal{M}, \Gamma\rangle$ if and only if $\pi^* \circ \phi^*$ is an optimal joint policy for $\mathcal{M}^+$.
\end{observation}
\begin{proof}
We show that $\hat{\rho}(\sigma_h) = \hat{R}_h(\sigma_h, \phi^*)$ under the independence condition, which makes the error bound in Theorem~\ref{thm:deterministic_equivalent} zero.
Let $\sigma_h$ be the plan-time sufficient statistic that maximizes the error.
Continue from Eq.~\eqref{eq:reward_lowerbound_b}, and apply the fact that $\sigma_h(\vec{z}_{-i,h}\mid\vec{z}_{i,h}) = \sigma_h(\vec{z}_{-i,h})$ under the independence condition:
\begin{subequations}
\begin{align}
\hat{R}_h&(\sigma_h, \phi^*) \!\!= \!\!\frac{1}{n}\sum\limits_{i=1}^{n}\sum\limits_{\vec{z}_{i,h}} \sigma_h(\vec{z}_{i,h}) \!\!\!\!\max\limits_{a_{i,h} \in A_{i,h}}\!\!\sum\limits_{\vec{z}_{-i,h}} \sigma_h(\vec{z}_{-i,h})\sum\limits_{s_h}\sigma_h(s_h\!\mid\! \vec{z}_{-i,h}, \vec{z}_{i,h})R_{i,h}(s_h, a_{i,h})\\
&=\frac{1}{n}\sum\limits_{i=1}^{n}\sum\limits_{\vec{z}_{i,h}, \vec{z}_{-i,h}} \sigma_h(\vec{z}_{i,h})\sigma_h(\vec{z}_{-i,h}) \max\limits_{a_{i,h} \in A_{i,h}}\sum\limits_{s_h}\sigma_h(s_h\mid \vec{z}_{-i,h}, \vec{z}_{i,h})R_{i,h}(s_h, a_{i,h})\\
&=\frac{1}{n}\sum\limits_{i=1}^{n}\sum\limits_{\vec{z}_{h}} \sigma_h(\vec{z}_{h}) \max\limits_{a_{i,h} \in A_{i,h}}\sum\limits_{s_h}\sigma_h(s_h\mid \vec{z}_{h})R_{i,h}(s_h, a_{i,h})\\
&=\frac{1}{n}\sum\limits_{i=1}^{n}\sum\limits_{\vec{z}_{h}} \sigma_h(\vec{z}_{h}) \max\limits_{\alpha \in \Gamma}\sum\limits_{s_h}\sigma_h(s_h\mid \vec{z}_{h})\alpha(s_h)\\
&=\frac{1}{n}\sum\limits_{i=1}^{n} \hat{\rho}(\sigma_h) = \hat{\rho}(\sigma_h),
\end{align}
\end{subequations}
where the second equality follows since the effect of the expectation over $\vec{z}_{-i,h}$ is the same for every $\vec{z}_{i,h}$.
Therefore $V_{\langle \mathcal{M}, \Gamma \rangle}^* = V_{\mathcal{M}^+}^{\pi^* \circ \phi^*}$, and the claim follows.
\end{proof}

\section{Further details on the APAS algorithm} 
\label{sec:further_details_on_the_apas_algorithm}

The APAS algorithm from the main paper is reproduced in Algorithm~\ref{alg:apas}.
We provide additional details on how we draw samples from the plan-time sufficient statistic, and how the $\alpha$-vectors are computed.

To avoid explicitly computing the plan-time sufficient statistic $\sigma_h$ in the adaptation phase (Lines~\ref{line:adapt_start}-\ref{line:adapt_end}), we instead apply rollouts to sample $\vec{z}_h \sim \sigma_h$ as follows.
We sample an initial state $s_0 \sim b_0$, and then simulate the Dec-POMDP by taking actions prescribed by $\pi_{best}$, sampling the next states $s_{t+1}$ and joint observations $z_{t+1}$ from the dynamics function $T$ until $t=h-1$.
We obtain a sampled sequence $\vec{z}_h$ of joint observations, and a sequence $\vec{a}_h$ of joint actions taken.
Then, we apply Bayesian filtering to compute the joint state estimate $b_k$ that is equal to $\sigma_h(\cdot \mid \vec{z}_h)$, i.e., $b_k(s_h) \triangleq \mathbb{P}(s_h \mid \vec{z}_h, \vec{a}_h, b_0)$.

The computation of the tangent hyperplane $\alpha_k$ of $f$ at the joint state estimate $b_k$ is based on the convex conjugate or Fenchel conjugate $f^*$.
We present here a brief overview, details are found in~\cite[Sect.~3.3.1.]{boyd2004convex}.
We derive the special case where $f$ is the negative entropy, but the procedure can be applied to any bounded, convex and differentiable function $f$.
Fix a linearization point $b_k\in \Delta(S)$.
Because $f$ is convex and differentiable, the following inequality holds for any $b\in \Delta(S)$:
\begin{equation}
\label{eq:approx}
f(b) \geq b^T \left[\nabla f(b_k) - f^*(\nabla f(b_k)) \right].
\end{equation}
In the problems we consider in the main paper, $f$ is the negative entropy: $f(b) = \sum_{s}b(s)\ln b(s)$ with $\nabla f(b) = \ln b + 1$.\footnote{For a vector $b$, expressions such as $\ln b$ denote the vector obtained taking the element-wise log of $b$.}
The Fenchel conjugate of $f$ is the log-sum-exp function $f^*(b) = \ln\left(\sum_{s} e^{b(s)} \right)$.
We see that $f^*(\nabla f(b_k)) = \ln\left(\sum_{s} e^{\ln b(s) + 1} \right) = \ln\left(e \sum_{s} b(s)  \right) = \ln (\sum_{s} b(s)) + \ln e = 1$, since $\sum_s b(s) = 1$.
By plugging these values to Eq.~\eqref{eq:approx} we obtain
\begin{equation}
f(b) \geq b^T \left[\nabla f(b_k) - f^*(\nabla f(b_k)) \right] = b^T \left( \ln b_k + 1 - 1 \right) = b^T \ln b_k = \sum\limits_{s} b(s) \ln b_k(s),
\end{equation}
from which we identify the $\alpha$-vector $\alpha_k(s) = \ln b_k(s)$.

A reference implementation of APAS as described here is available in a public repository hosted at \url{https://github.com/laurimi/multiagent-prediction-reward}.

\section{Details of the experimental setup} 
\label{sec:experimental_setup}
We present additional details on the algorithms, parameter settings, and the experimental setup we apply.

\subsection{Overview of the solution algorithms} 
\label{sub:overview_of_the_solution_algorithms}
We briefly review the NPGI solution algorithm for Dec-$\rho$POMDPs~\citep{Lauri:2019:IGD:3306127.3331815}, and the policy graph improvement method of~\cite{Pajarinen2011} that we use to implement the planning subroutine of our proposed APAS method.
The two solution algorithms are closely related, motivating selecting them for our comparison.

\paragraph{NPGI.}
NPGI~\citep{Lauri:2019:IGD:3306127.3331815} represents each agent's policy using a finite state controller (FSC) with a fixed number of controller states.
Since NPGI solves a finite horizon problem, each node is identified with a particular time step $t$ in the problem.
A conceptual example of a FSC policy is represented in Figure~\ref{fig:fsc}.
Each controller state is labelled by an individual action $a_{i,t}$ to be taken.
At each controller state, a transition function determines the next controller node for each individual observation $z_{i,t+1}$.
The transition function is represented by the edges of the directed graph shown in the figure.

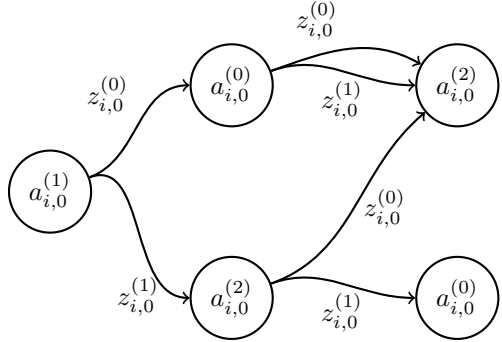
\begin{figure}[t]
  \centering 
  \begin{tikzpicture}[auto, node distance=2cm]
    \node[circle, draw, thick] (start) {$a_{i,0}^{(1)}$};
    \node[circle, draw, thick, above right of=start, xshift=1cm] (q1) {$a_{i,0}^{(0)}$};
    \node[circle, draw, thick, below right of=start, xshift=1cm] (q2) {$a_{i,0}^{(2)}$};

    \draw[thick, ->] (start) to[out=20,in=180] node {$z_{i,0}^{(0)}$} (q1);
    \draw[thick, ->] (start) to[out=20,in=180] node [below, yshift=-0.5cm] {$z_{i,0}^{(1)}$} (q2);
    \node[circle, draw, thick, right of=q1, xshift=1cm] (q3) {$a_{i,0}^{(2)}$};
    \node[circle, draw, thick, right of=q2, xshift=1cm] (q4) {$a_{i,0}^{(0)}$};

    \draw[thick, ->] (q1) to[out=20,in=150] node {$z_{i,0}^{(0)}$} (q3);
    \draw[thick, ->] (q1) to[out=20,in=180] node [below] {$z_{i,0}^{(1)}$} (q3);

    \draw[thick, ->] (q2) to[out=20,in=220] node [right] {$z_{i,0}^{(0)}$} (q3);
    \draw[thick, ->] (q2) to[out=20,in=180] node [below] {$z_{i,0}^{(1)}$} (q4);
  \end{tikzpicture}
  \caption{An illustration of a finite state controller representing an individual policy $\pi_i$ of agent $i$ for horizon $h=3$. Controller states represented by the circular nodes. There are 2 controller states per time step, except for $t=0$ where only the leftmost starting node is present. There are three possible individual actions $a_{i,t}^{(j)}$ at $t=0, 1, 2$.
  There are two possible individual observations $z_{i,t}^{(j)}$ at $t=1, 2$.}
    \label{fig:fsc}
\end{figure}

NPGI optimizes the actions to take and the transition function of the FSC of each agent by repeating two phases.
First, for each controller state, NPGI computes the expected joint state estimate, marginalizing over the possible controller states of the other agents $-i$.
Secondly, NPGI solves a local optimization problem at each controller state, finding an improved action and an improved transition function.

The first step of NPGI requires computing \emph{all joint state estimates reachable under the current policy}, increasing the complexity of the algorithm.
NPGI solves Dec-$\rho$POMDPs with reward functions that are convex functions of the joint state estimate.
This means that it is difficult to adapt it to use sampling-based approximations or rollouts instead of explicit computation of joint state estimates.
In our experiments, we use the implementation of NPGI provided by the authors of~\cite{Lauri:2019:IGD:3306127.3331815}.

\paragraph{Policy graph improvement.}
NPGI is a generalization of the policy graph improvement algorithm of~\cite{Pajarinen2011}.
The policy graph improvement algorithm also represents each agent's policy as a FSC with a fixed number of nodes, and operates using the same two phases as NPGI.

For the \textsc{Plan} subroutine of APAS, we implement the finite-horizon policy graph improvement algorithm based on the description provided in~\cite{Pajarinen2011}.
We modify the algorithm to use sampling and rollouts to estimate values and expected joint state estimates.

\subsection{Conversion to standard Dec-POMDP: implementation details}
The conversion from Dec-$\rho$POMDP to a Dec-POMDP increases the horizon by one, and on the newly added time step modifies the action and observation spaces, the transition and observation models, and the reward function.
The actions on the newly added time step are the individual prediction actions.
The transition and observation models and the reward function on the newly added time step are trivial.

We found it easiest to handle the conversion implicitly, that is, we load the horizon $h$ Dec-$\rho$POMDP description into memory, and then instruct our solution algorithm to solve the horizon $(h+1)$ Dec-POMDP while implementing the modifications mentioned above on the final time step directly in the solver.
This implicit conversion is fast and its effect on the overall solution time is negligible.

We also experimented with first loading the problem from disk in the \texttt{.dpomdp} format~\cite{oliehoek2017madp}, then modifying the description to include the changes required before writing it back to disk, thereby explicitly creating the converted standard Dec-POMDP.
However, we found this approach to be infeasible, as the \texttt{.dpomdp} format is not straightforward to use with time-dependent action and observation spaces and reward functions.

\subsection{Parameter settings} 
\label{sub:implementation_details_and_hyperparameter_settings}
As described in the previous subsection, the planning algorithms used in our experiments are closely related.
We therefore share the parameters for both of them.
In all our experiments, we use FSCs with 2 nodes per time step (see Fig.~\ref{fig:fsc} for an example of the resulting FSCs).
20 policy improvement iterations are executed.
We escape local maxima by assigning a random action and a random transition function for a node under optimization with probability 0.1.

\subsection{Experimental settings} 
\label{sub:experimental_settings}
For APAS and NPGI, we execute 10 runs using each and record in each run the value of the best joint policy found.
For APAS without the adaptation phase, using only a single set of randomly sampled $\alpha$-vectors, we instead run 100 runs.
We terminate any run that does not finish within a timeout of 2 hours.
Since all algorithms we use are anytime algorithms, we sometimes can obtain results from a partially finished run as well.
All experiments were run on a computer with an Intel Core i7-5930K CPU, with 32 GB of memory.

\section{Additional experimental results} 
\label{sec:experimental_results}
In this section, we present additional experimental results omitted from the main paper.

\paragraph{Adaptation phase of APAS significantly improves performance.}
\begin{figure}
    \includegraphics[width=0.33\columnwidth]{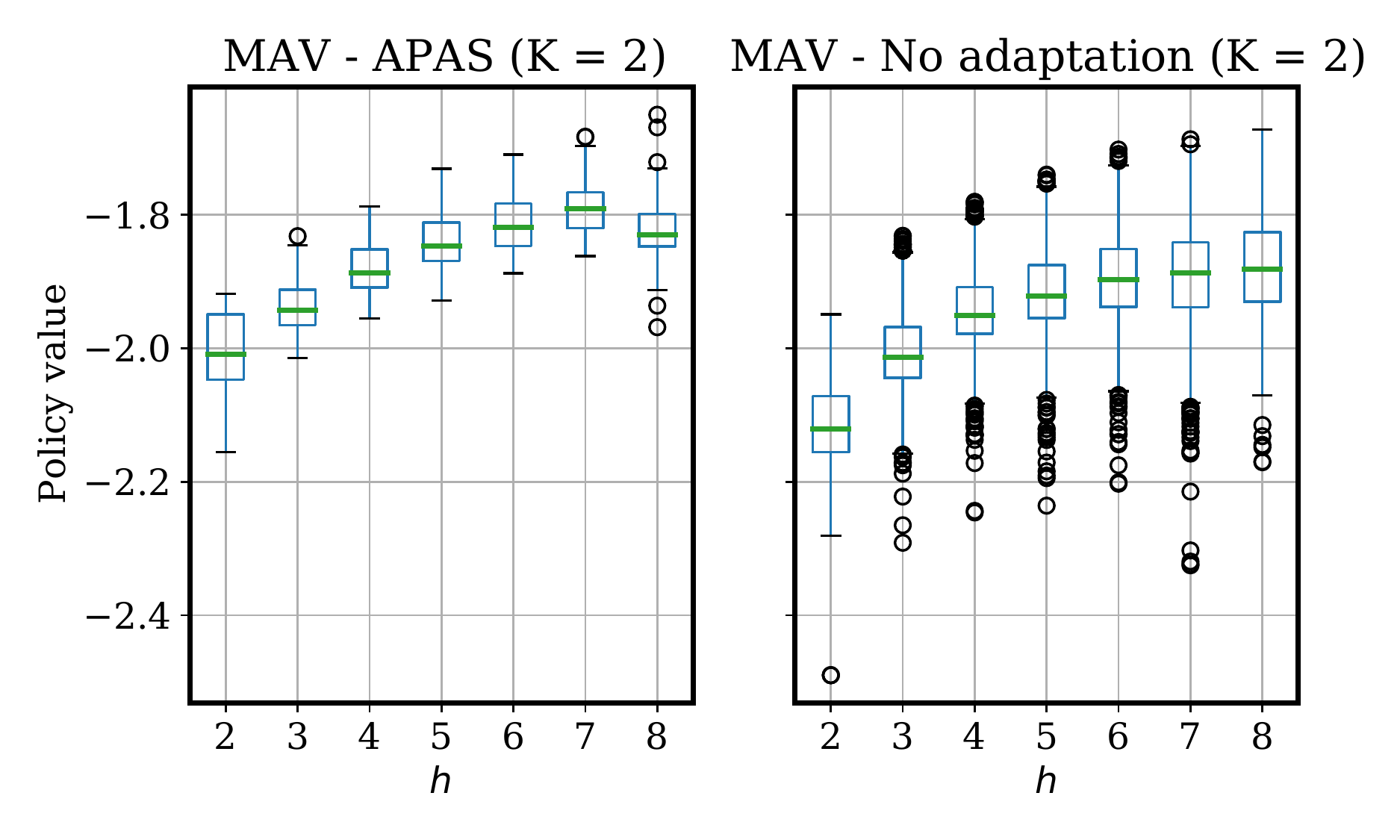}~
    \includegraphics[width=0.33\columnwidth]{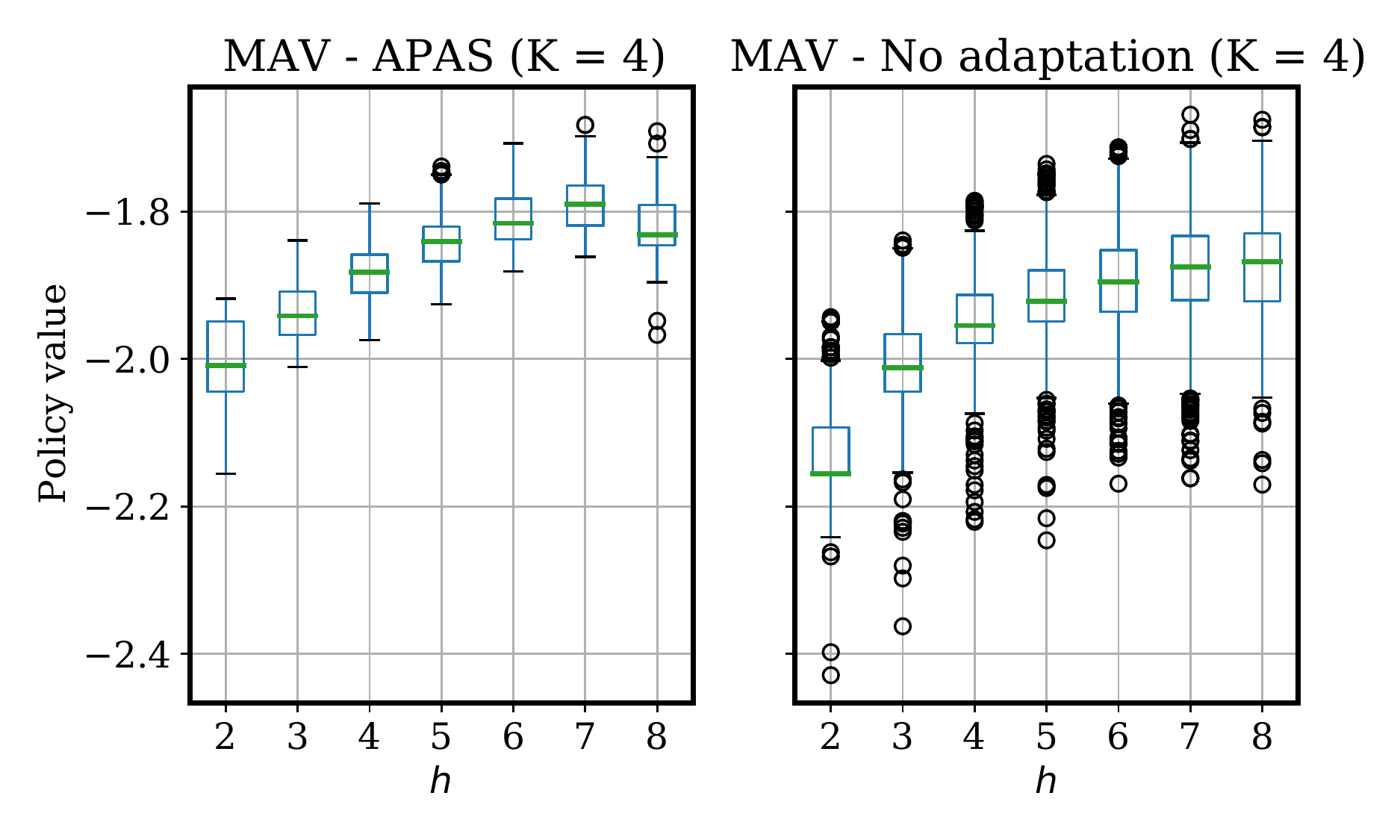}~
    \includegraphics[width=0.33\columnwidth]{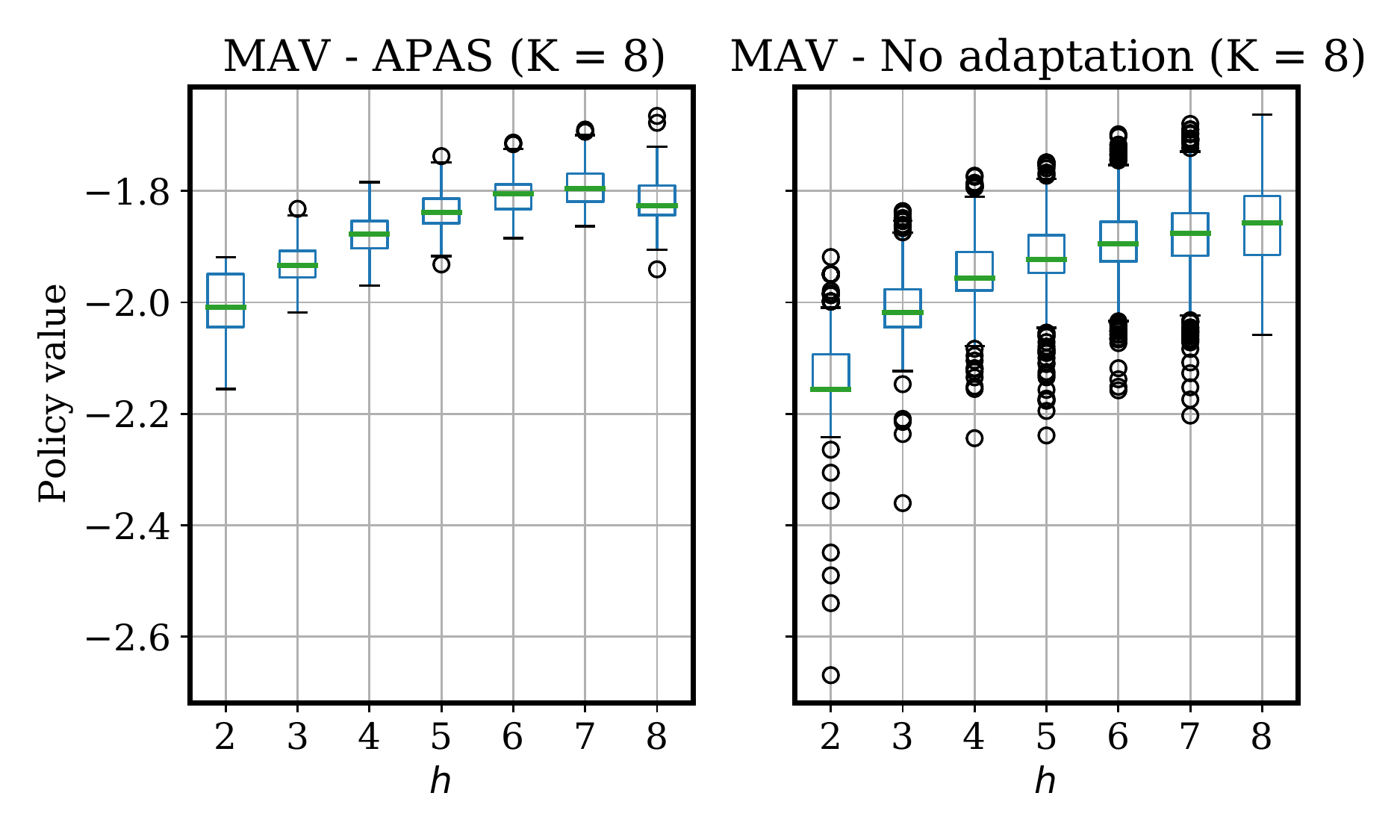}~
  \caption{Boxplots of policy values in the MAV domain found by APAS and APAS without the adaptation phase as a function of the horizon $h$. The plots are arranged in three groups of two plots.
  From left to right, the groups of two plots report results for $K=2$, 4, or 8 individual prediction actions.
  Within each group, the left plot reports the result for APAS, and the right plot the result for APAS without adaptation.}
  \label{fig:adaptation_comparison_mav}
\end{figure}
\begin{figure}
    \includegraphics[width=0.5\columnwidth]{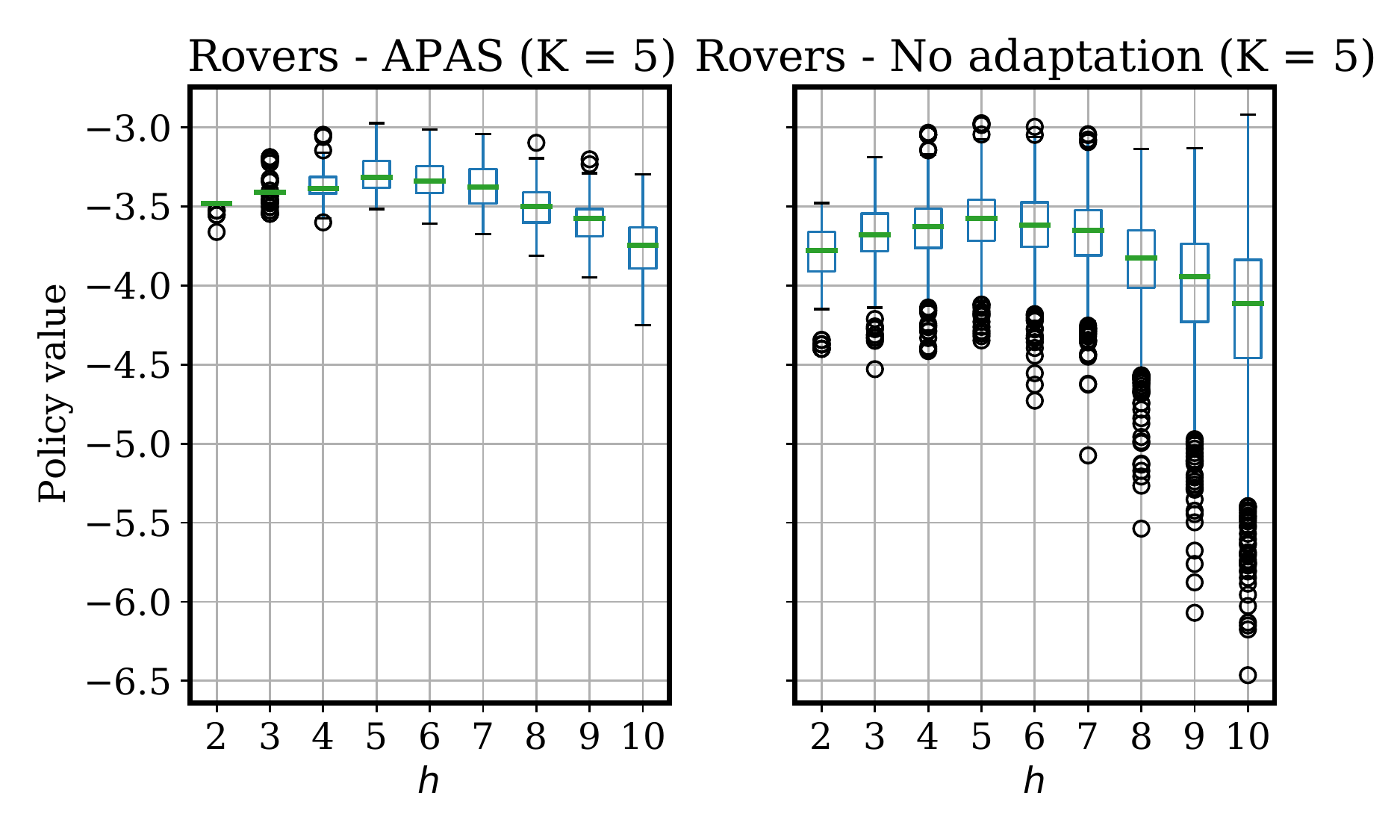}~
    \includegraphics[width=0.5\columnwidth]{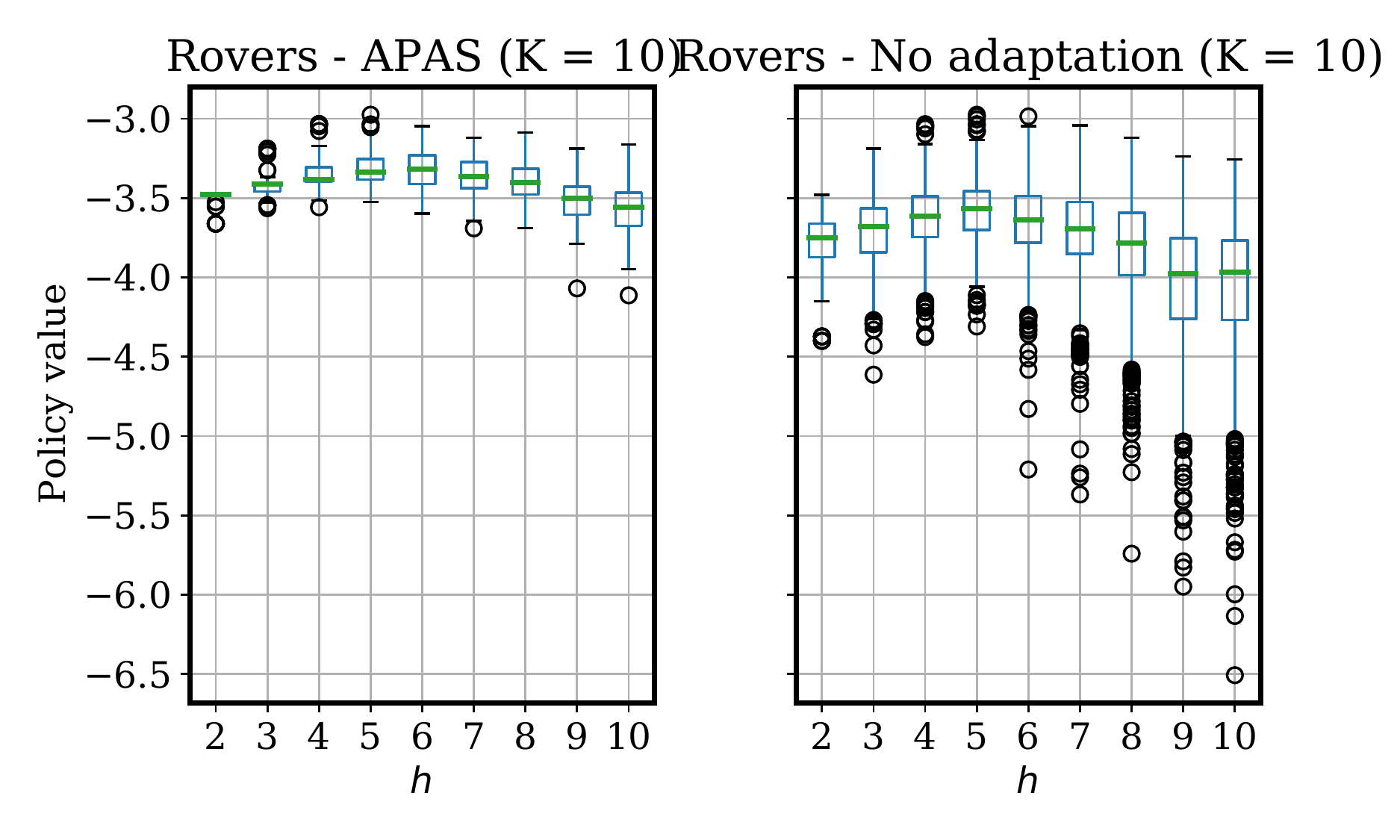}
  \caption{Boxplots of policy values in the Rovers domain found by APAS and APAS without the adaptation phase as a function of the horizon $h$.
  From left to right, the groups of two plots report results for $K=5$ or 10 individual prediction actions.
  Within each group, the left plot reports the result for APAS, and the right plot the result for APAS without adaptation.}
  \label{fig:adaptation_comparison_rovers}
\end{figure}
We present here more detailed results on removing the adaptation phase from APAS (Algorithm~\ref{alg:apas}).
Instead of executing the adaptation phase, we randomly sample the linearization points and corresponding $\alpha$-vectors.
Figures~\ref{fig:adaptation_comparison_mav} and~\ref{fig:adaptation_comparison_rovers} shows a comparison of policy values between APAS and APAS without the adaptation phase for the MAV and Rovers domains, respectively.
We see that including the adaptation phase consistently improves the average value of policies found in both problem domains.
Generally policy values are higher for APAS, while the variance is also lower, indicating usefulness of the adaptation phase.
Notably, the worst case performance is much better with the adaptation phase than without it.

We also note that as the horizon increases, both methods experience a decrease in average performance.
The greatest policy values found without adaptation sometimes exceed the values of policies found by APAS, e.g., for horizon $h=10$ with $K=5$ in the Rovers domain.
This suggests that further improvements to APAS might be possible by improving the adaptation phase.

\paragraph{Effect of the number of $\alpha$-vectors.}
The effect of the number $K$ of $\alpha$-vectors (number of individual prediction actions) on the performance of APAS is shown in Figure~\ref{fig:apas_k_mav} for the MAV domain and in Figure~\ref{fig:apas_k_rovers} for the rovers domain.
Each subplot shows for a particular planning horizon $h$ boxplots of the values of policies found by APAS as a function of $K$.
We observe that only in the Rovers domain for $h=10$ using $K=10$ individual prediction actions compared to $K=5$ results in slightly improved performance, although the difference is not very significant.

\begin{figure}
  \begin{tabular}{ccc}
    \includegraphics[width=0.3\textwidth]{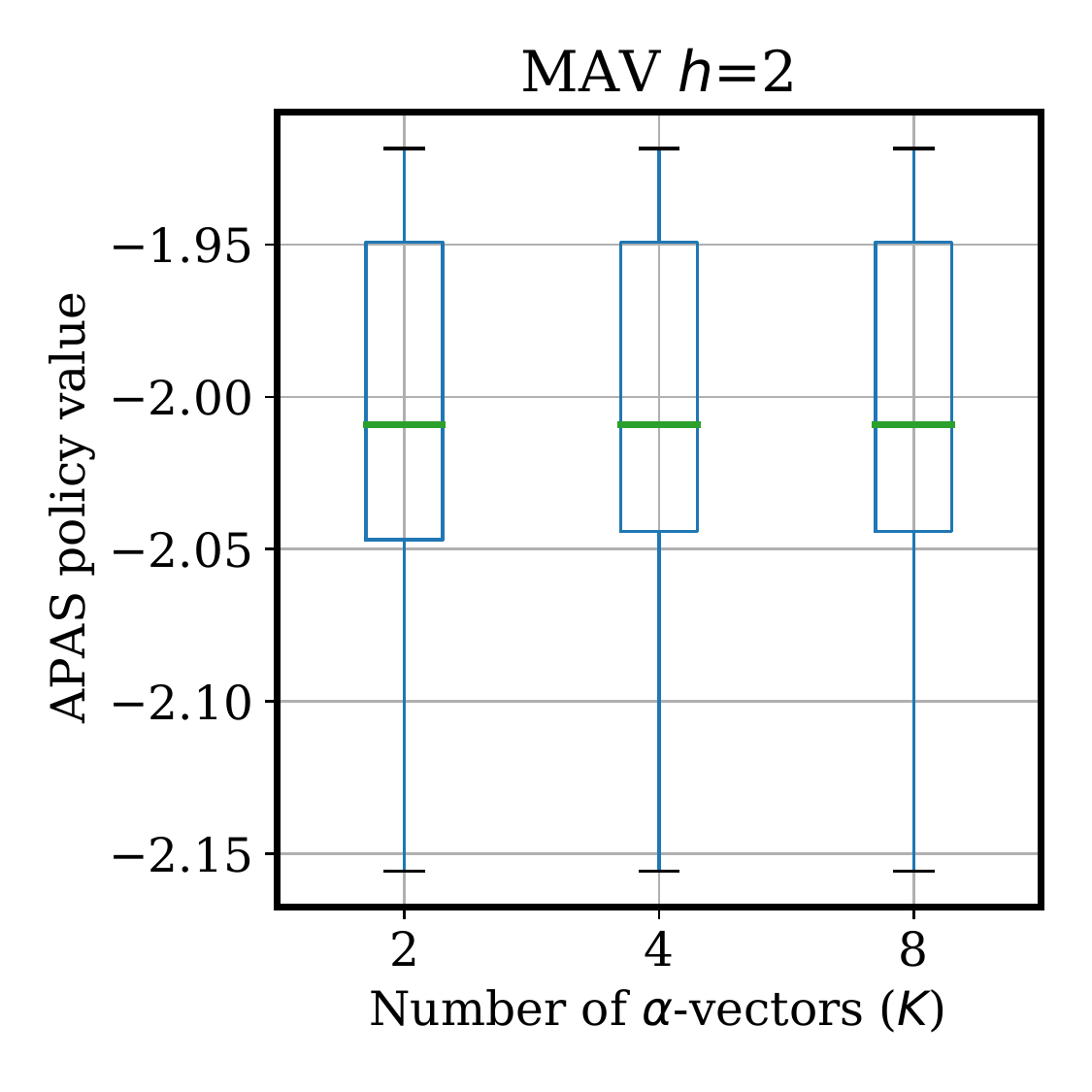} & \includegraphics[width=0.3\textwidth]{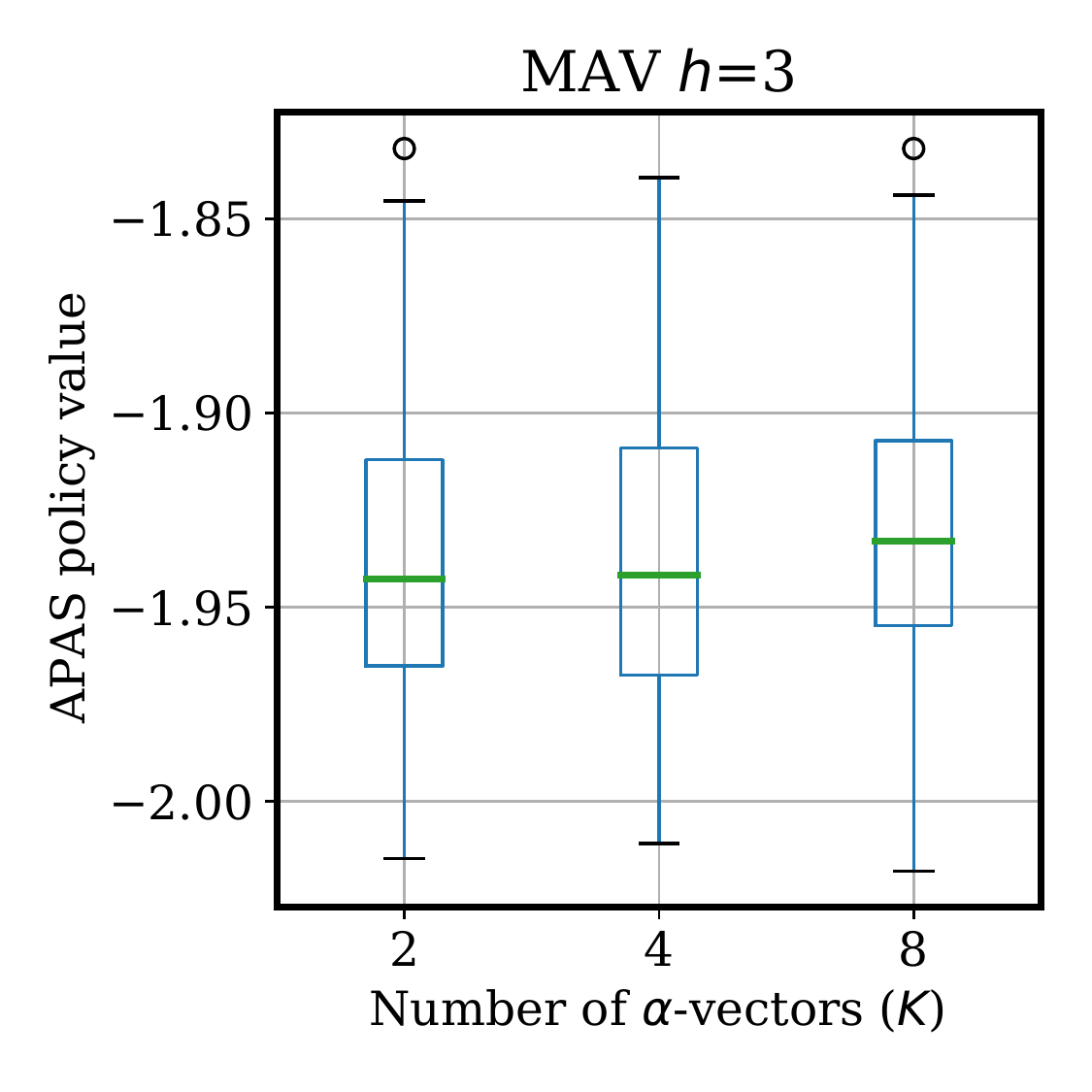} & \includegraphics[width=0.3\textwidth]{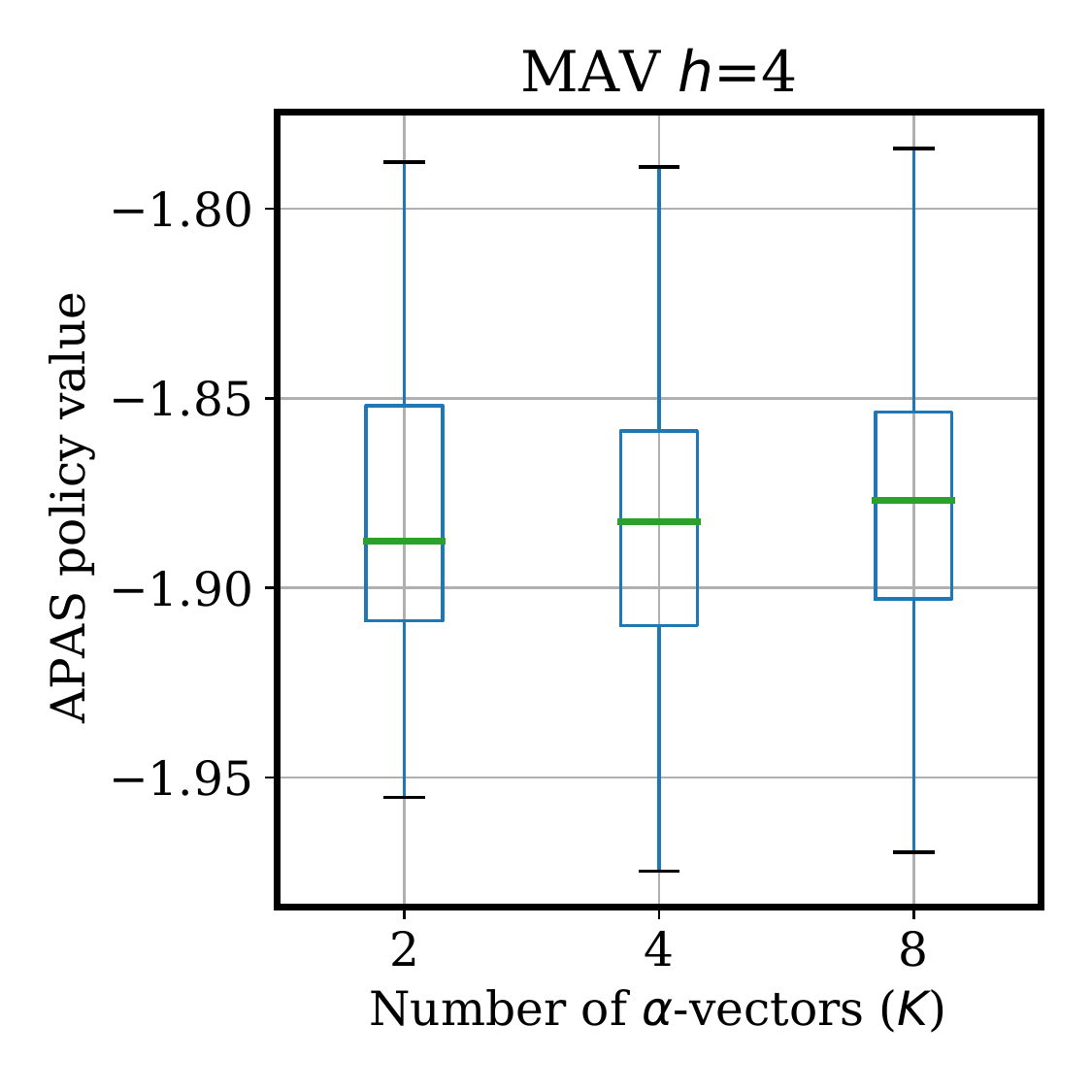} \\
    \includegraphics[width=0.3\textwidth]{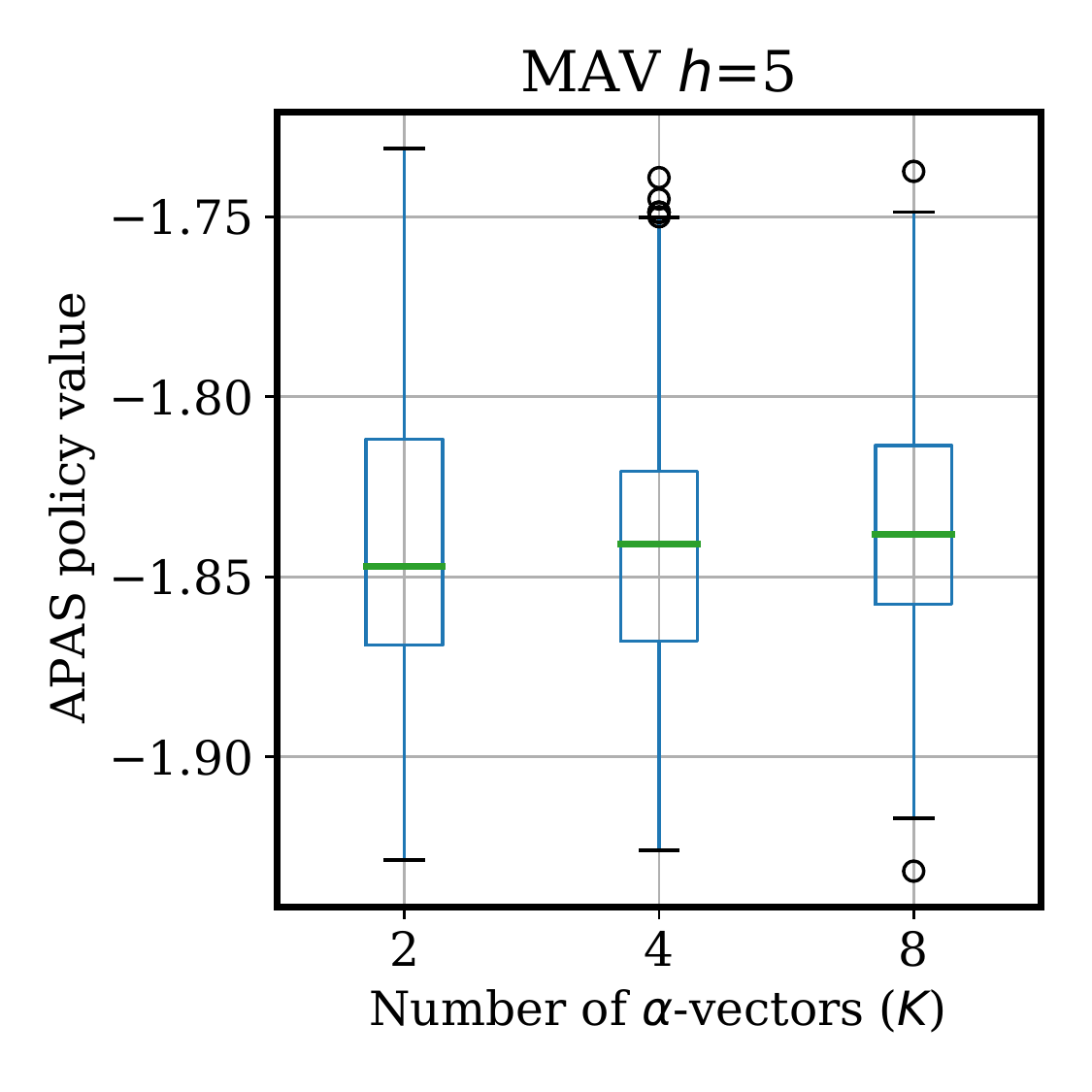}& \includegraphics[width=0.3\textwidth]{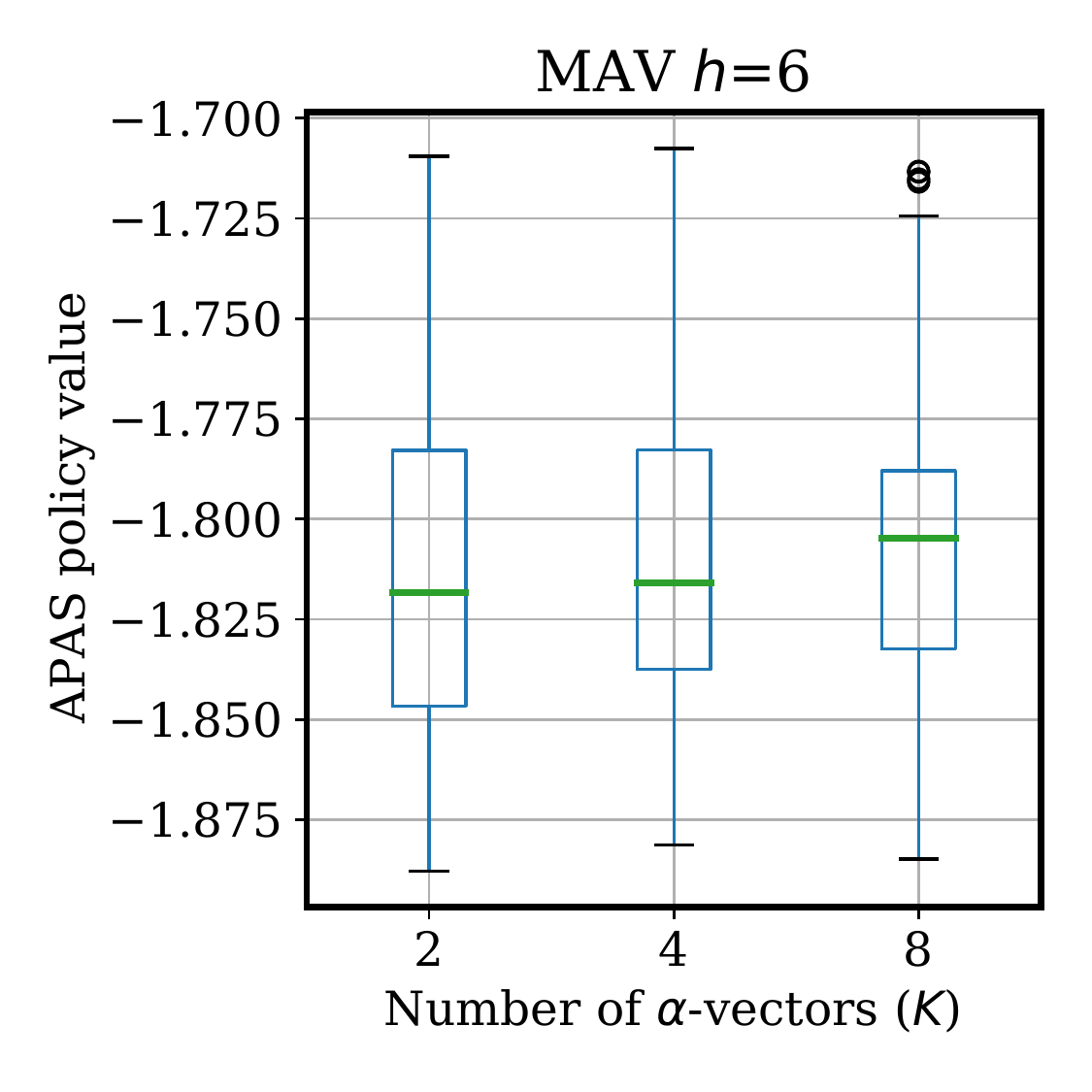}& \includegraphics[width=0.3\textwidth]{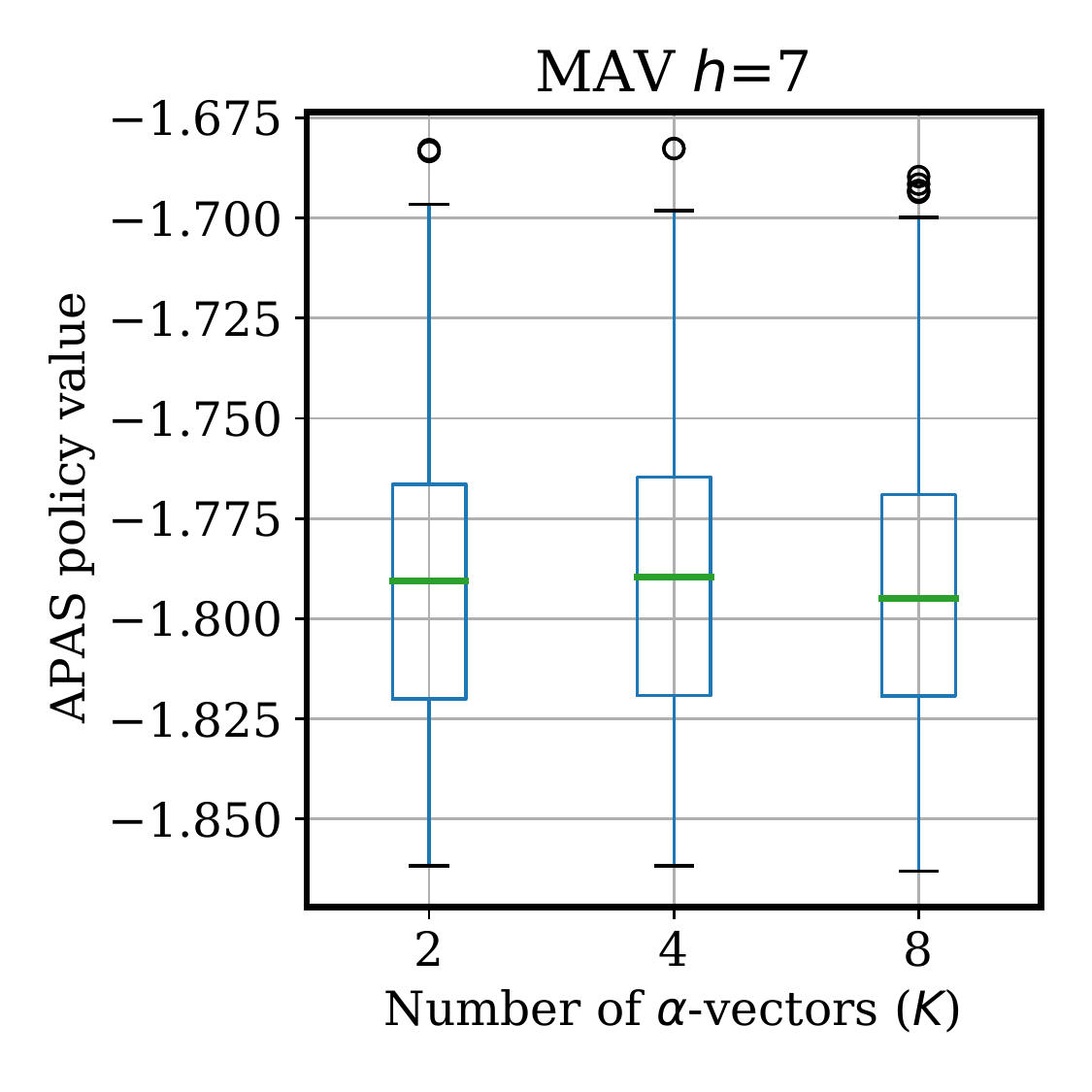} \\
    & \includegraphics[width=0.3\textwidth]{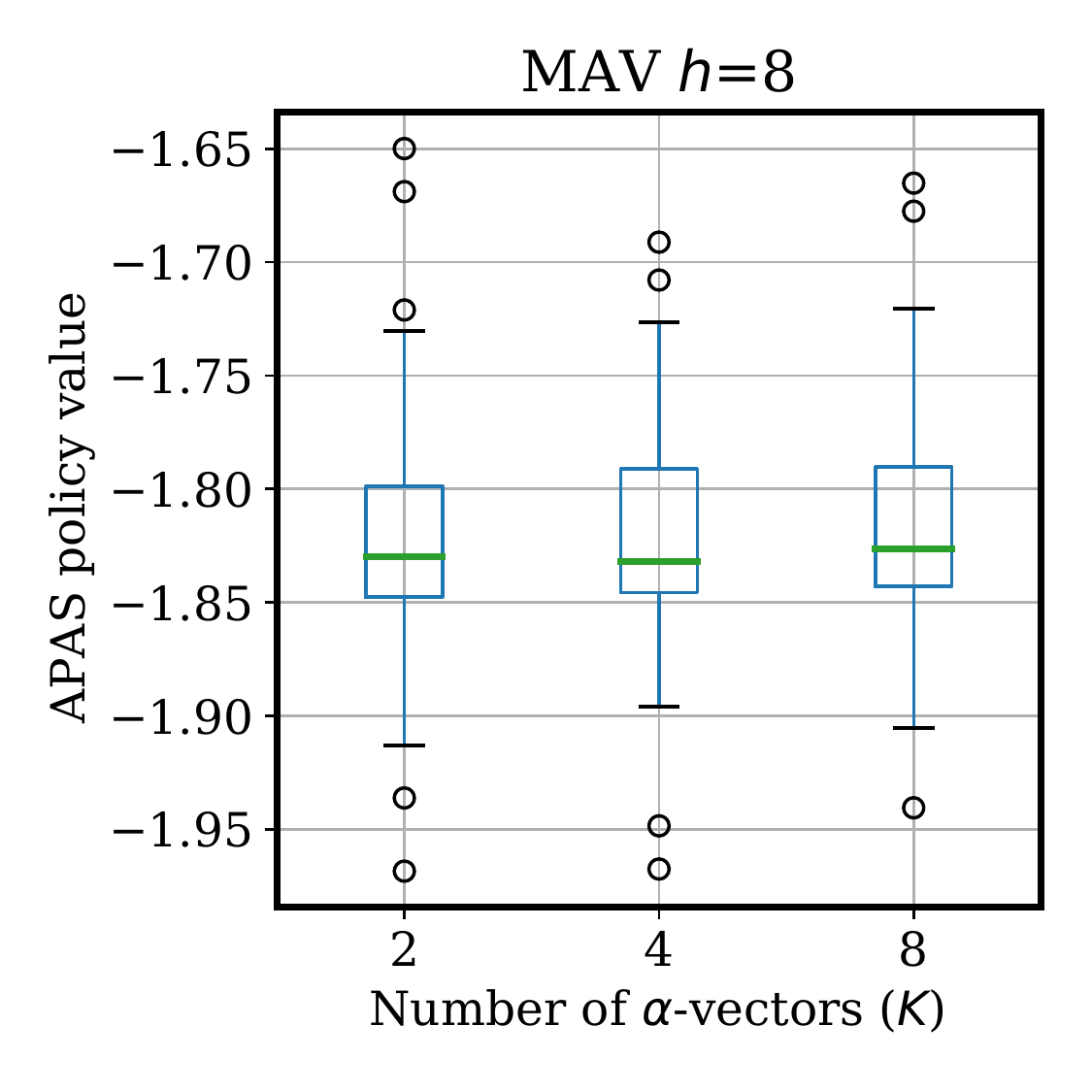} & 
  \end{tabular}
  \caption{Boxplots of values of policies found by APAS in the MAV domain as a function of the number $K$ of $\alpha$-vectors.}
  \label{fig:apas_k_mav}
\end{figure}

\begin{figure}
  \begin{tabular}{ccc}
    \includegraphics[width=0.3\textwidth]{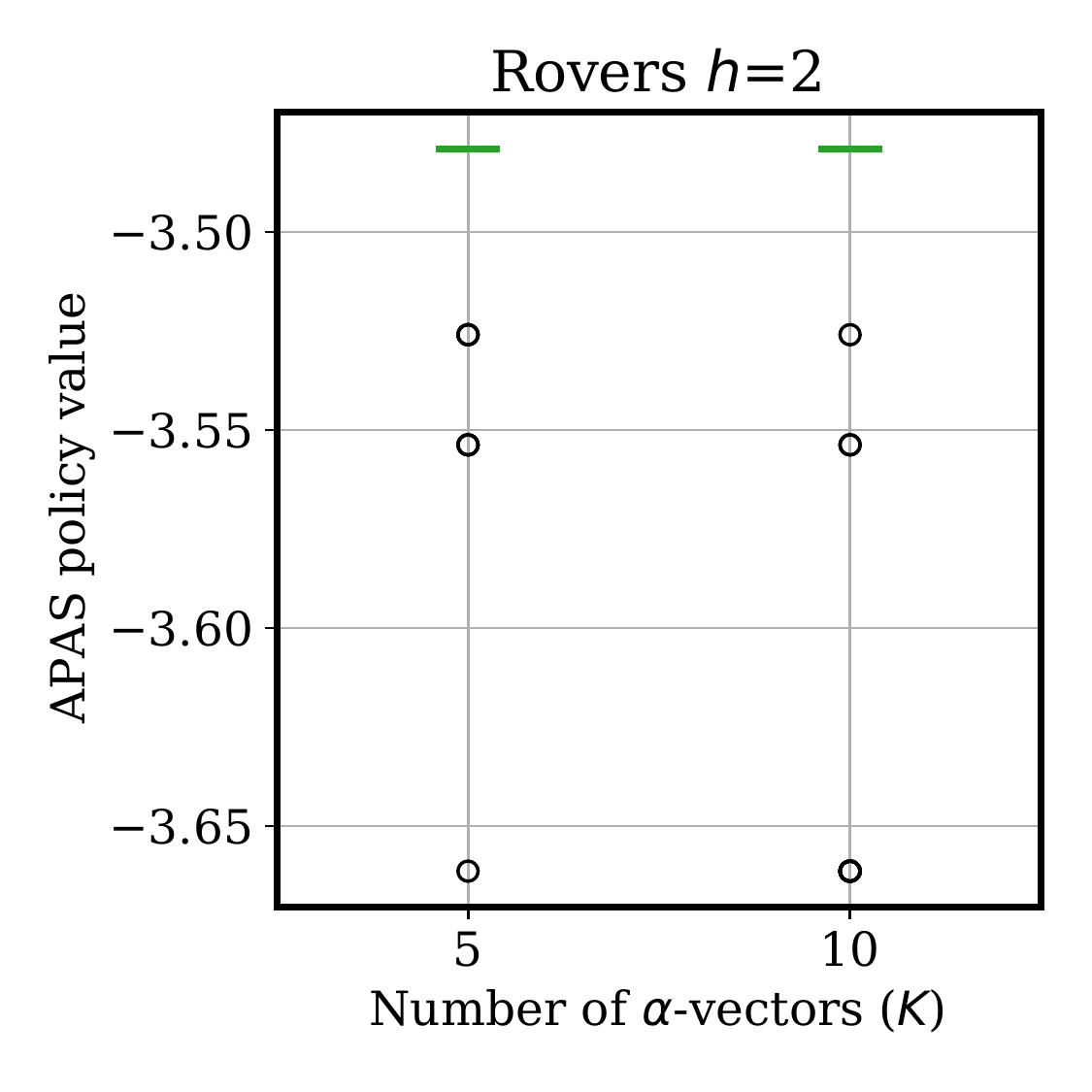} & \includegraphics[width=0.3\textwidth]{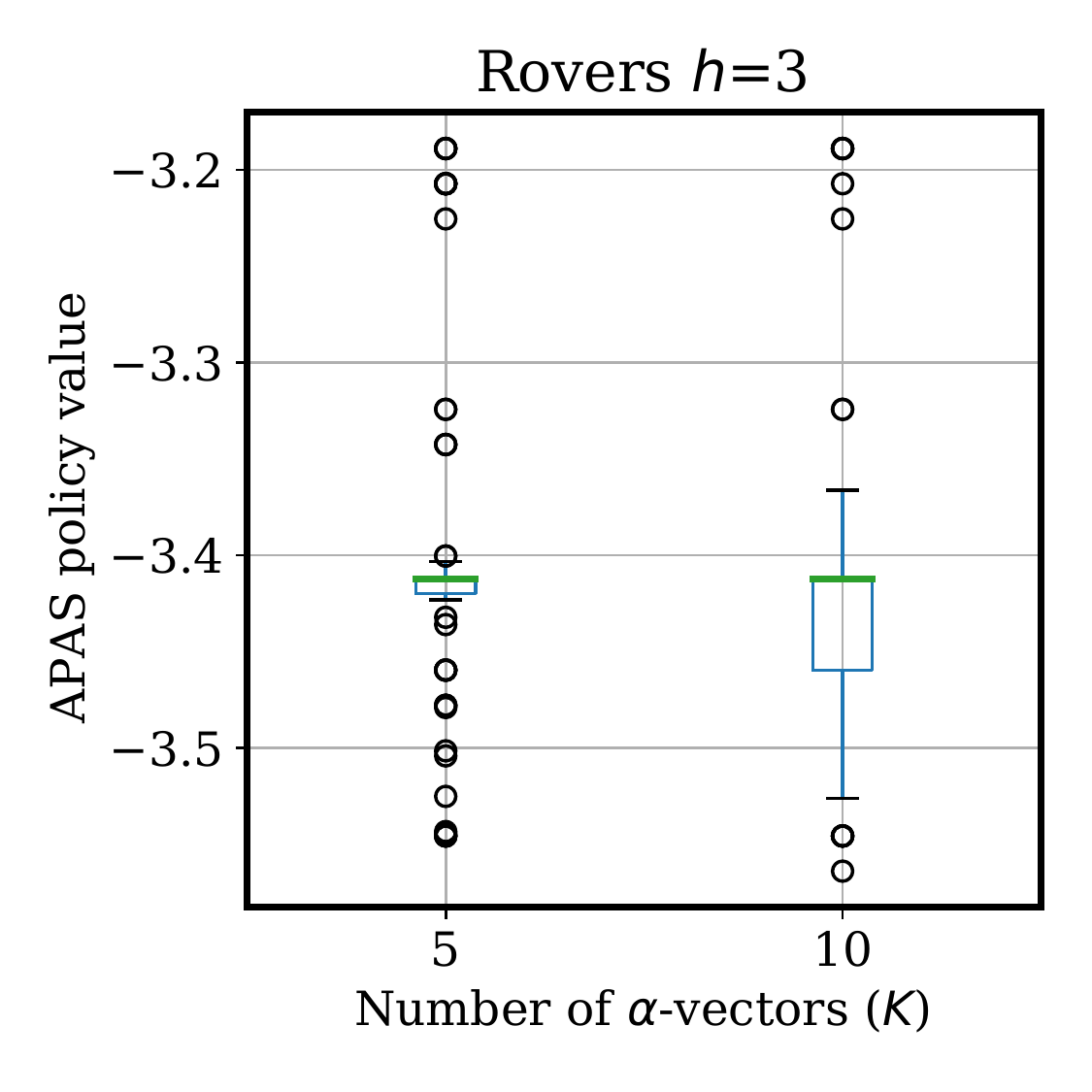} & \includegraphics[width=0.3\textwidth]{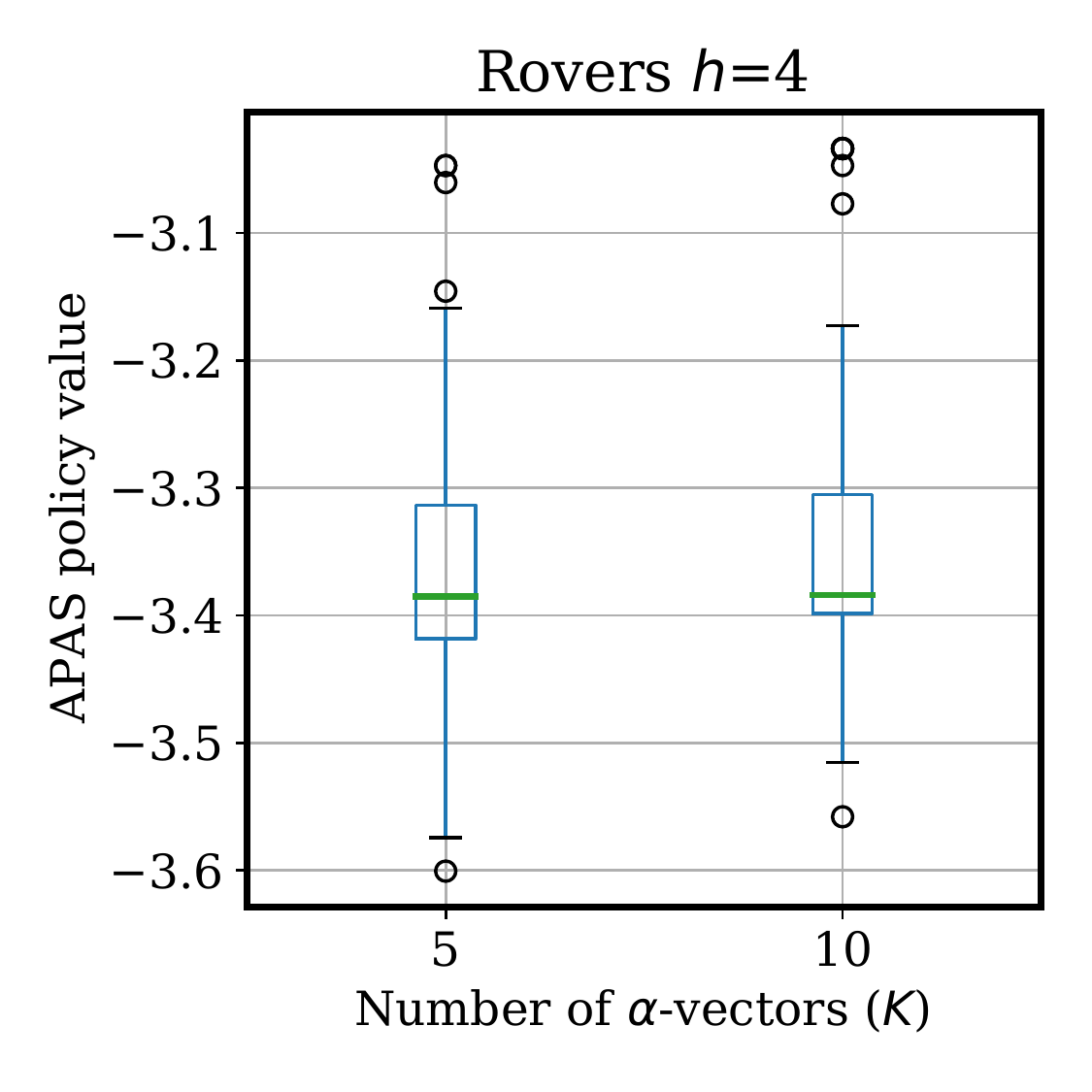} \\
    \includegraphics[width=0.3\textwidth]{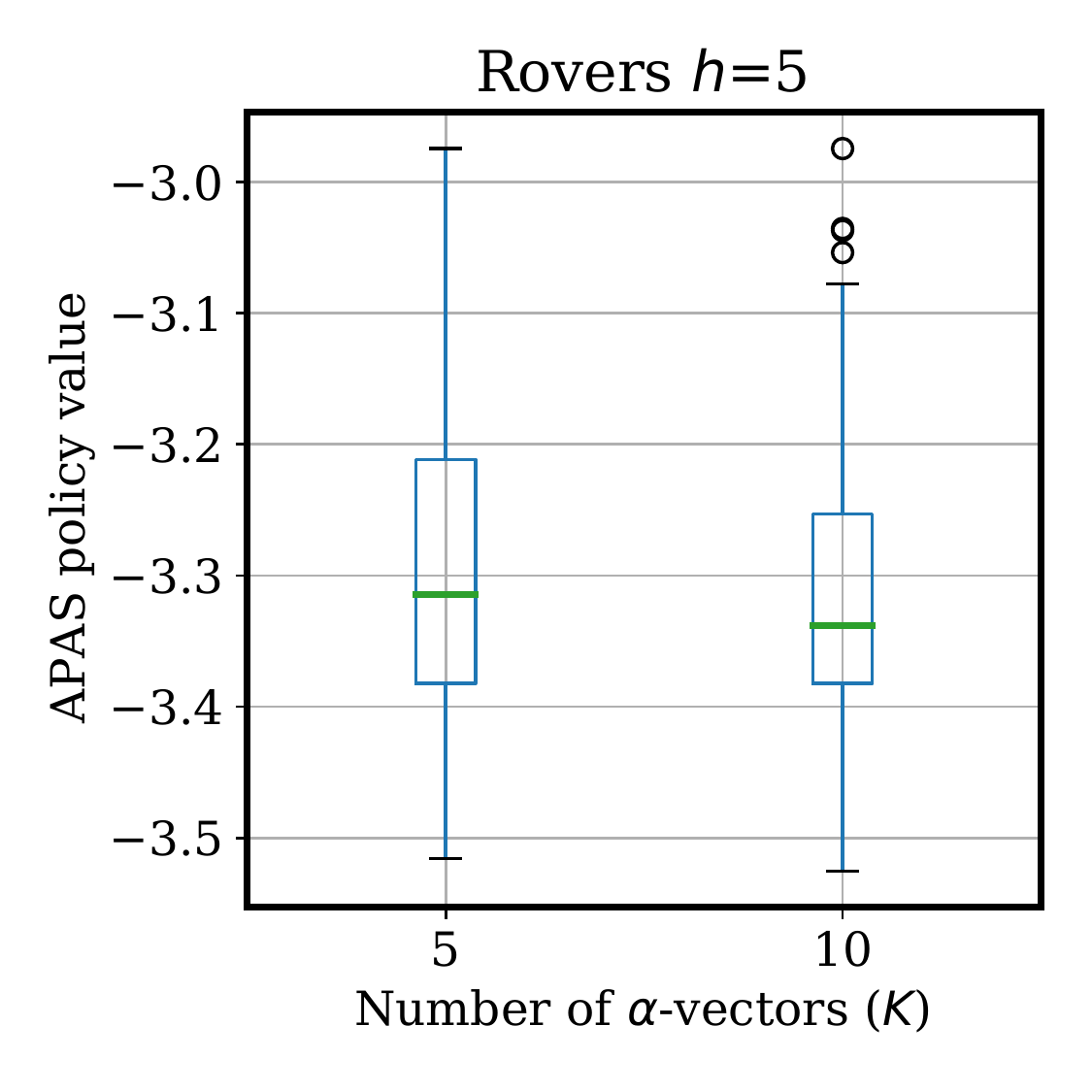}& \includegraphics[width=0.3\textwidth]{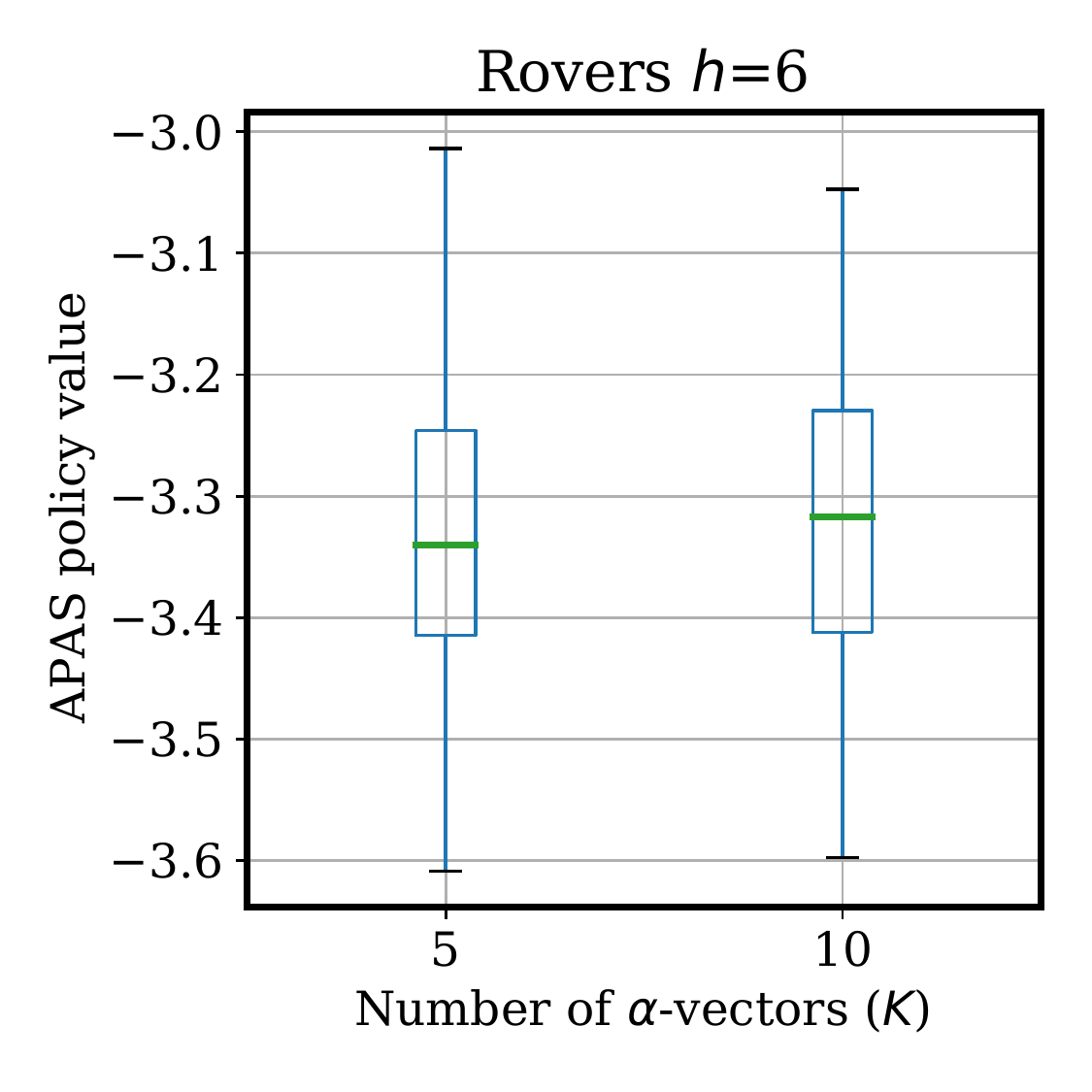}& \includegraphics[width=0.3\textwidth]{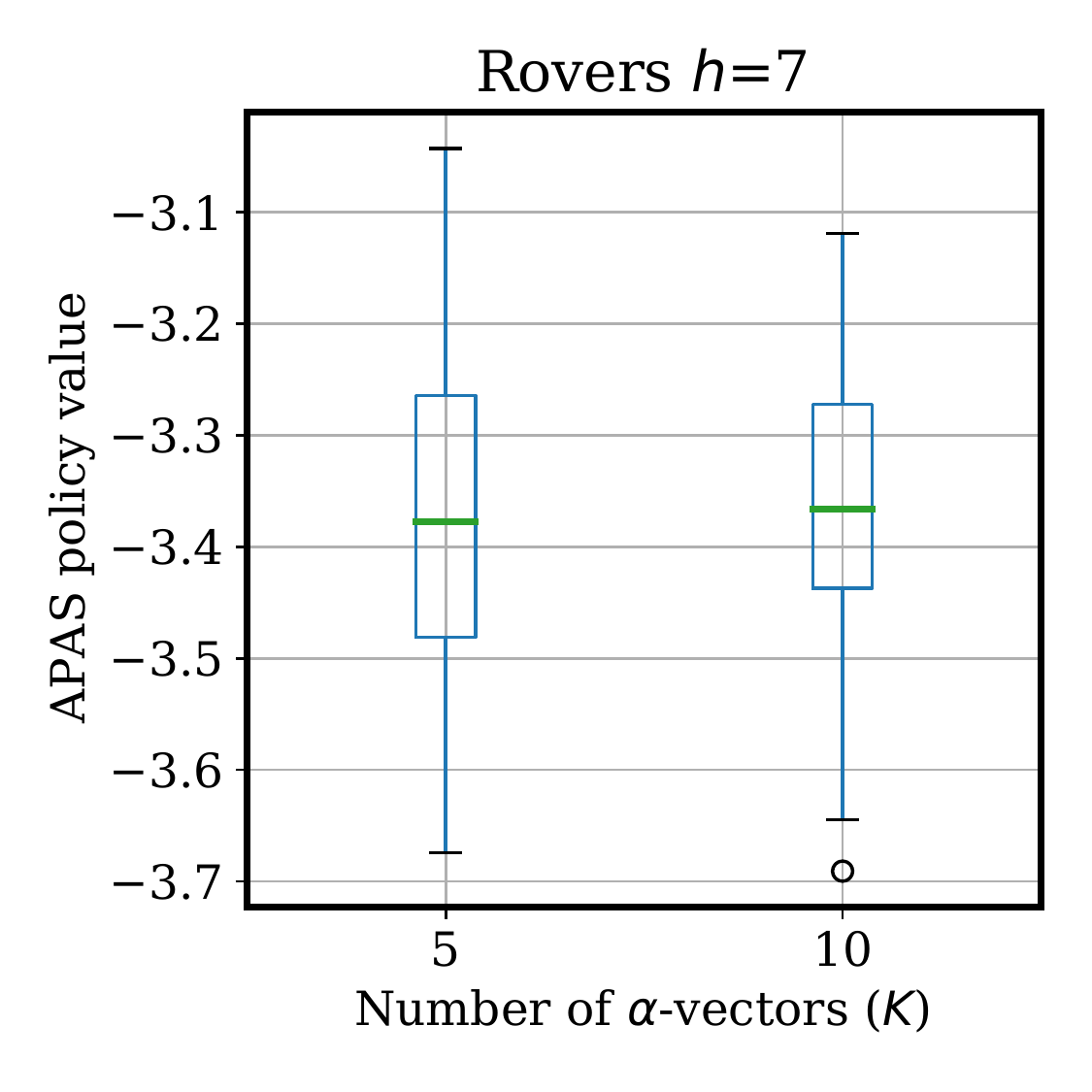} \\
    \includegraphics[width=0.3\textwidth]{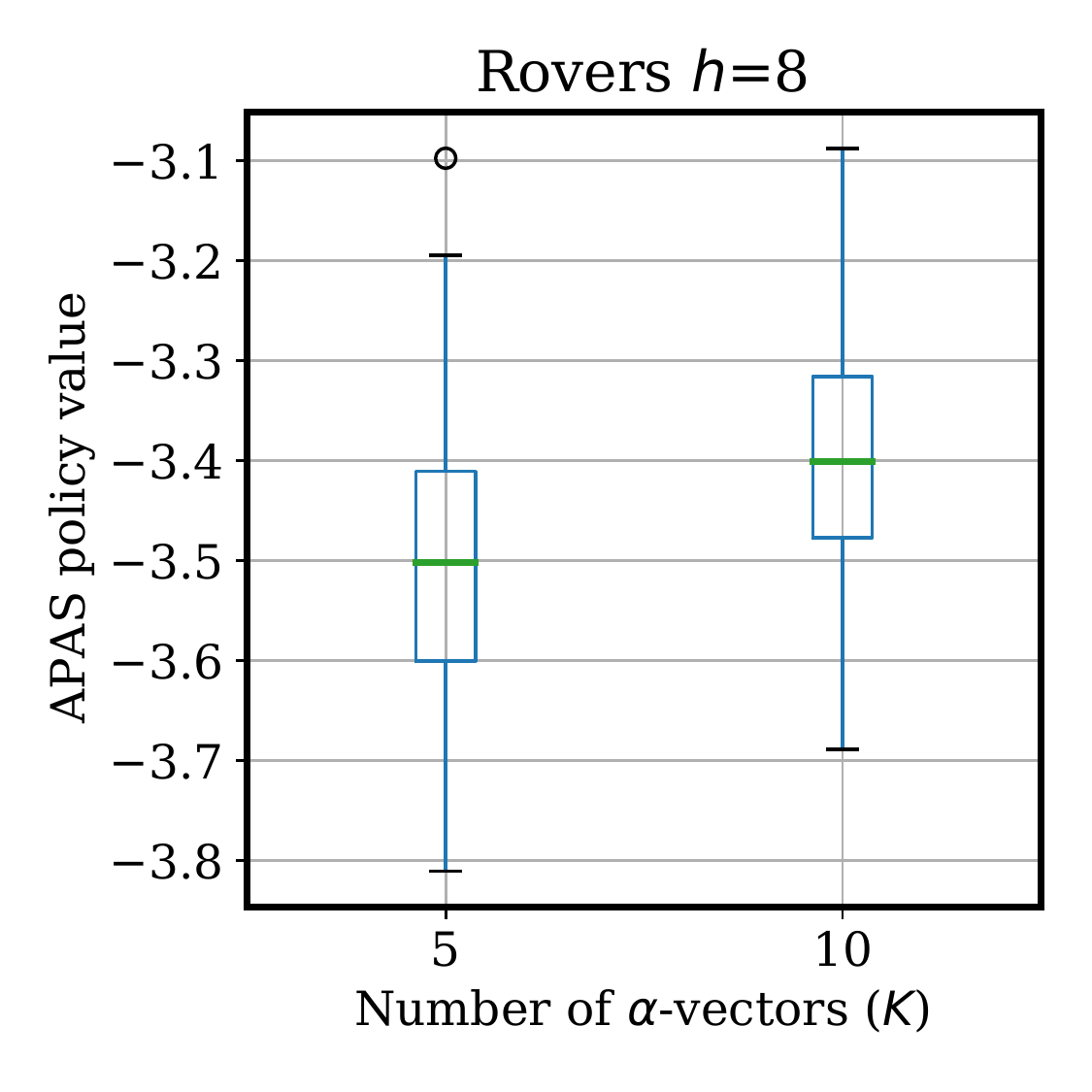}& \includegraphics[width=0.3\textwidth]{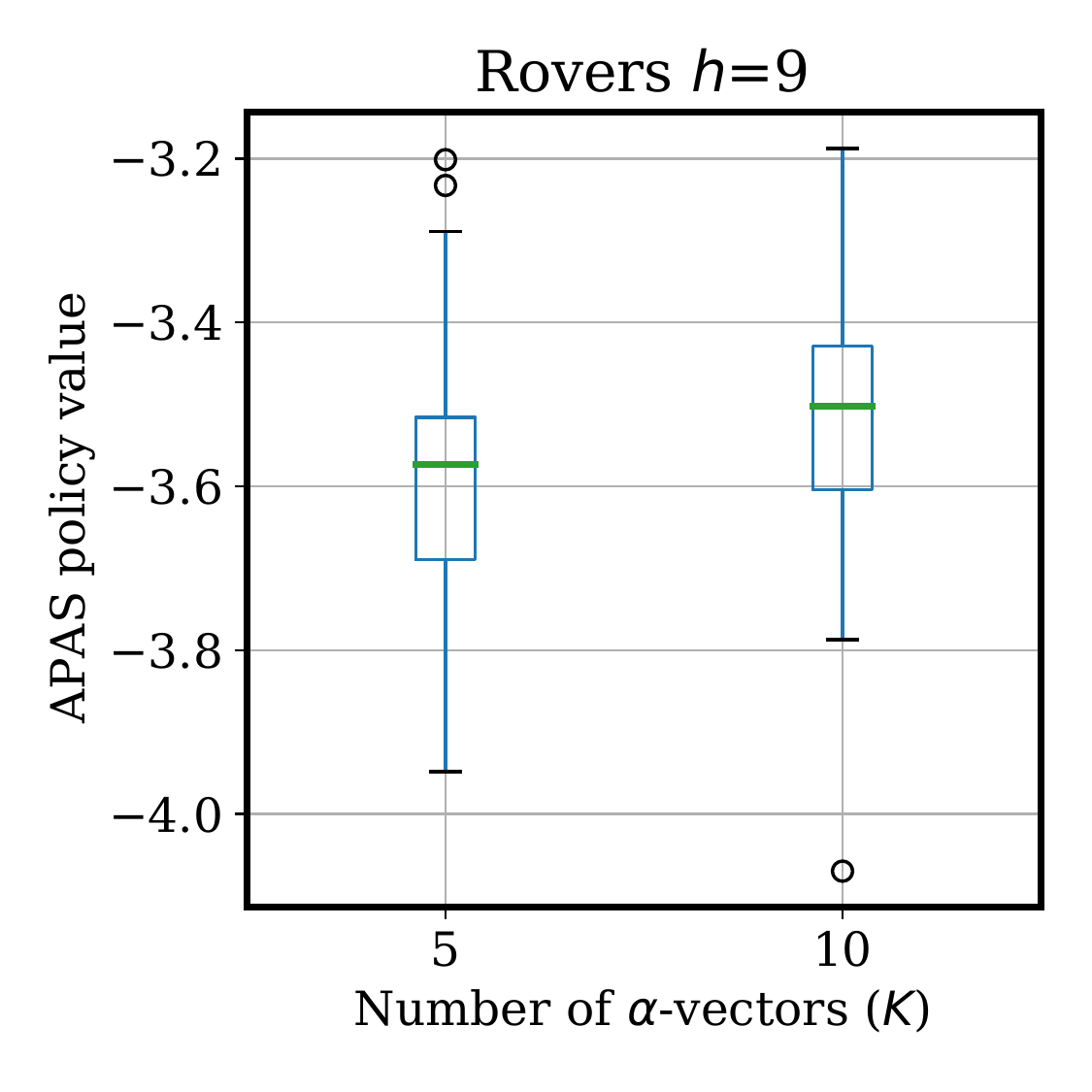}& \includegraphics[width=0.3\textwidth]{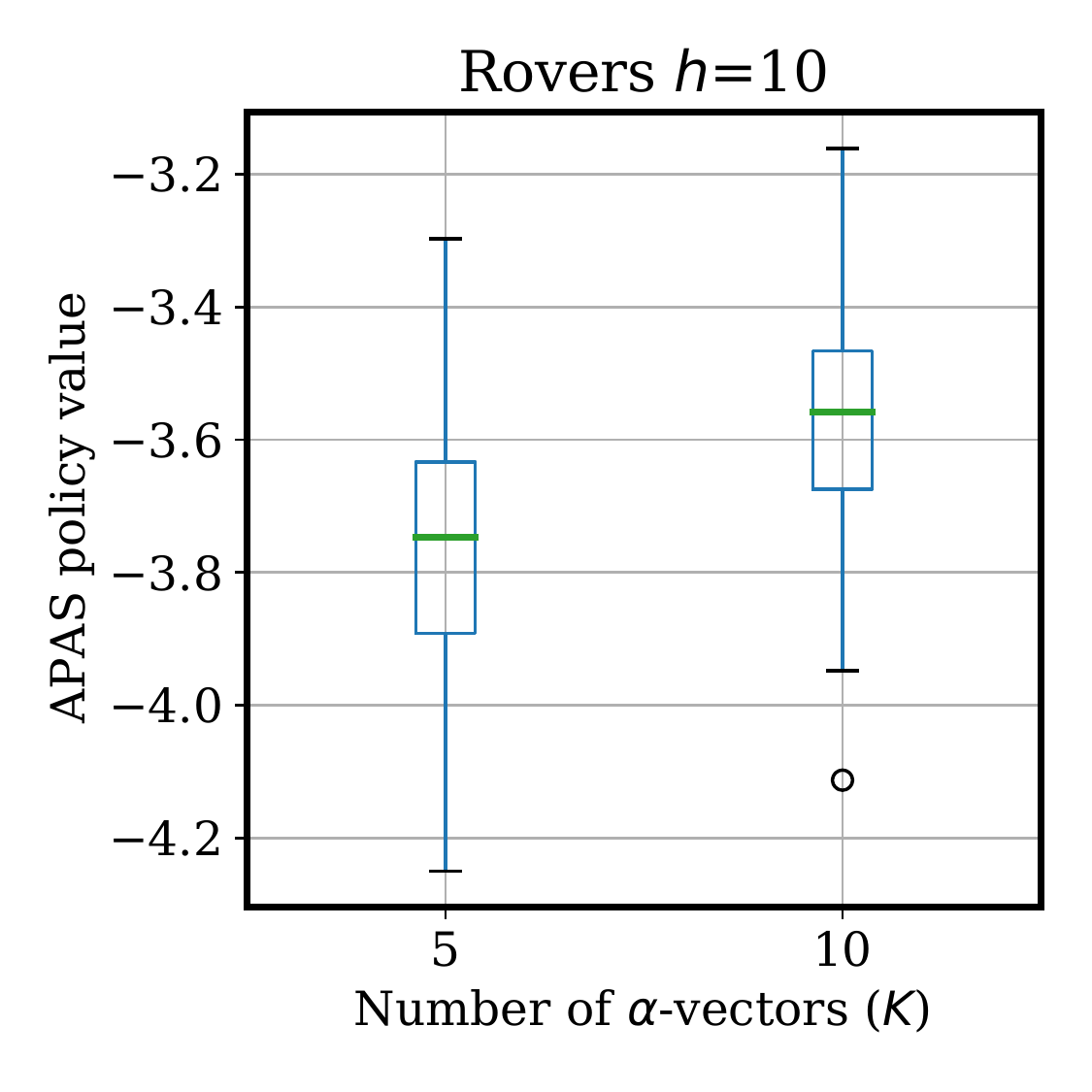} 
  \end{tabular}
  \caption{Boxplots of values of policies found by APAS in the Rovers domain as a function of the number $K$ of $\alpha$-vectors (number of individual prediction actions).}
  \label{fig:apas_k_rovers}
\end{figure}

\paragraph{Timing results.}
\begin{table}[t]
\centering
\caption{Average total runtime of APAS $\pm$ standard deviation in seconds in the MAV domain $(K=2)$ and the Rovers domain $(K=5)$.}
\label{tab:timing}
\begin{tabular}{@{}ccc@{}}
\toprule
Horizon $h$  & MAV (seconds) & Rovers (seconds) \\ \midrule
2  & 1.44 $\pm$ 0.15    & 21.05 $\pm$ 0.26       \\
3  & 2.01 $\pm$ 0.12    & 21.87 $\pm$ 0.28       \\
4  & 2.73 $\pm$ 0.14    & 22.96 $\pm$ 0.48       \\
5  & 3.99 $\pm$ 0.18    & 24.33 $\pm$ 0.58       \\
6  & 14.59 $\pm$ 0.21   & 27.39 $\pm$ 0.91       \\
7  & 190.14 $\pm$ 1.65  & 35.18 $\pm$ 2.62       \\
8  & 2909.11 $\pm$ 334.29 & 65.31 $\pm$ 4.63       \\
9  & ---    & 177.51 $\pm$ 19.41      \\
10 & ---    & 607.16 $\pm$ 90.82       \\ \bottomrule
\end{tabular}
\end{table}

Table~\ref{tab:timing} shows the average total duration of APAS runs for the MAV domain and Rovers domain using $K=2$ and $K=5$ prediction actions, respectively.
The average runtime increases strongly as the horizons increases.
Upon inspection, we noted that the majority of the time for long horizons was spent evaluating the policies (Line~\ref{line:evaluate} of Algorithm~\ref{alg:apas}).
We evaluate policies exactly, by computing all reachable state estimates to evaluate the final reward $f$ at them.
This suggests that further scaling in terms of the planning horizon is possible if switching to an approximate evaluation of policy value, e.g., by sampling and evaluating trajectories.
We did not explore the effect of the resulting noisy value estimates on the solution quality.
Somewhat surprisingly, Table~\ref{tab:timing} indicates lower runtimes for horizons $h\geq 7$ for the larger Rovers domain.
This is due to the domain structure.
In any state, most observations in Rovers have zero probability as the agent always correctly observes its own location.
These zero-probability observations do not need to be considered in value computation.

Finally, we can compare the time required by APAS and NPGI.
While a direct comparison is not fully informative due to differing implementations, we note that in the MAV domain with $h=5$ a runtime of around 30 seconds per \emph{single backward pass} is reported for NPGI in~\cite{Lauri_JAAMAS2020}.
As shown in Table~\ref{tab:timing}, for the same domain and horizon, APAS completes the planning (which in our case includes \emph{20 backward passes}) in about 4 seconds.

\section*{References}
\bibliographystyle{plainnat}
\begingroup
\renewcommand{\section}[2]{}%
\small
\bibliography{refs}
\endgroup

\end{document}